\documentclass[letterpaper]{article} 
\usepackage{aaai2027}  
\usepackage[hyphens]{url}  
\usepackage{graphicx} 
\usepackage{natbib}  
\usepackage{caption} 
\usepackage{algorithm}
\usepackage{algpseudocode}
\usepackage{url}            
\usepackage{amsfonts}       
\usepackage{nicefrac}       
\usepackage{microtype}      
\usepackage{xcolor}         
\usepackage{enumitem} 
\usepackage{amsmath} 
\setcitestyle{numbers,square}
\setlist[itemize]{noitemsep, topsep=0pt}
\usepackage{diagbox} 
\usepackage{float}
\usepackage{tablefootnote}
\usepackage{multirow}
\usepackage{multicol}
\usepackage{adjustbox}
\usepackage{subfig}
\usepackage{longtable}
\usepackage{makecell}
\usepackage{amssymb}
\usepackage{amsthm}      
\usepackage{amssymb} 
\theoremstyle{plain}    
\newtheorem{remark}{Remark}
\newtheorem{theorem}{Theorem}
\newtheorem{proposition}{Proposition}

\theoremstyle{definition}
\newtheorem{definition}{Definition}
\newtheorem{assumption}{Assumption}

\usepackage{newfloat}
\usepackage{listings}
\DeclareCaptionStyle{ruled}{labelfont=normalfont,labelsep=colon,strut=off} 
\floatstyle{ruled}
\newfloat{listing}{tb}{lst}{}
\floatname{listing}{Listing}

\usepackage{booktabs}

\title{Bridging MARL to SARL: An Order-Independent \\ Multi-Agent Transformer via Latent Consensus}
\author{
    Zijian Zhao \textsuperscript{\rm 1},
    Jing Gao \textsuperscript{\rm 2} \thanks{Corresponding Author: Jing Gao},
    Sen Li \textsuperscript{\rm 1, \rm 3}
}
\affiliations{
    \textsuperscript{\rm 1}The Hong Kong University of Science and Technology
    \textsuperscript{\rm 2}The Hong Kong Polytechnic University
    \textsuperscript{\rm 3}The Hong Kong University of Science and Technology (Guangzhou)

}

\usepackage{bibunits}
\defaultbibliographystyle{aaai2027}
\defaultbibliography{aaai2027}

\begin{document}

\maketitle

\begin{bibunit}

\begin{abstract}
Cooperative multi-agent reinforcement learning (MARL) is widely used to address large joint observation and action spaces by decomposing a centralized control problem into multiple interacting agents. However, such decomposition often introduces additional challenges, including non-stationarity, unstable training, weak coordination, and limited theoretical guarantees. In this paper, we propose the Consensus Multi-Agent Transformer (CMAT), a centralized framework that reformulates fully cooperative, fully observable MARL as a hierarchical single-agent, multi-action reinforcement learning problem. CMAT treats all agents as a unified entity and employs a Transformer encoder to process the large joint observation space. To handle the extensive joint action space, we introduce a hierarchical decision-making mechanism in which a Transformer decoder autoregressively generates a high-level consensus vector, simulating the process by which agents reach agreement on their strategies in latent space. Conditioned on this consensus, all agents generate their actions simultaneously, enabling order-independent joint decision making and avoiding the sensitivity to action-generation order in conventional Multi-Agent Transformers (MAT). This factorization allows the joint policy to be optimized using single-agent PPO while preserving expressive coordination through the latent consensus. To evaluate the proposed method, we conduct experiments on benchmark tasks from StarCraft II, Multi-Agent MuJoCo, and Google Research Football. The results show that CMAT achieves superior performance over recent centralized solutions, sequential MARL methods, and conventional MARL baselines. The code for this paper is available at an anonymous repository: \url{https://anonymous.4open.science/r/CMAT-416A}.
\end{abstract}


\section{Introduction}

Cooperative Multi-Agent Reinforcement Learning (MARL) has become an important framework for solving complex real-world problems such as autonomous fleet coordination, traffic signal optimization, and robotic swarm control. In many of these settings, fully centralized control is feasible and sometimes even necessary; for example, ride-hailing order dispatch often requires global coordination to avoid redundant assignments. However, the joint observation and action spaces typically grow exponentially with the number of agents, giving rise to the Curse of Dimensionality (CoD) \cite{hady2025multi}. A common way to mitigate this issue is to decompose the original problem into multiple decentralized agents that learn collaboratively. Although this decomposition improves scalability, it also introduces substantial challenges, including non-stationarity, unstable training dynamics, and poor credit assignment, which can ultimately limit both empirical performance and theoretical guarantees \cite{jin2025comprehensive}.


A major line of research addresses these difficulties through the Centralized Training Decentralized Execution (CTDE) paradigm. Representative methods such as QMIX \cite{rashid2020monotonic}, Mean Field MARL \cite{yang2018mean}, and COMA \cite{foerster2018counterfactual} use a centralized critic to exploit global state information or neighborhood interactions during training, thereby alleviating instability and credit assignment issues. Nevertheless, under decentralized execution, agents still act independently, which restricts cooperation because each agent can only infer the behavior of others from past training experience. Similar limitations also arise in many implicit consensus-based MARL approaches \cite{oroojlooy2023review}. More recently, explicit communication mechanisms have been explored \cite{liexponential}, allowing agents to exchange information during both training and execution. While promising, these methods introduce additional design challenges, including how to choose communication partners and what information should be shared.

Motivated by these limitations, another line of research has shifted toward fully centralized solutions, namely Centralized Training Centralized Execution (CTCE). Among these approaches, the Multi-Agent Transformer (MAT) \cite{wen2022multi} has emerged as a representative framework by formulating cooperative MARL as a sequential decision-making problem. Specifically, MAT employs a centralized Transformer \cite{vaswani2017attention} encoder to capture relationships among all agents' observations and uses a decoder to generate actions autoregressively. However, the resulting policy can be highly sensitive to the action-generation order. Although recent studies have attempted to optimize this order jointly with action selection \cite{hu2025pmat,takayama2025aoad}, doing so substantially increases the complexity of the problem by expanding the search space to $n!$, where $n$ denotes the number of agents. In the same spirit, the recent Triple-BERT \cite{zhao2025triple} explores an alternative direction by modeling joint action probabilities directly and simultaneously. However, this method relies on a structured policy space that may limit action expressiveness, making it primarily suitable for trip-vehicle assignment problems.

To address these limitations, we propose the Consensus Multi-Agent Transformer (CMAT), which recasts fully cooperative, fully observable MARL as a hierarchical single-agent, multi-action reinforcement learning problem. Building upon MAT, CMAT replaces sequential action generation with an iterative consensus-generation process in the decoder, which simulates how agents reach agreement on their strategies in latent space.
Once the consensus vector is obtained, all agents generate their actions simultaneously while conditioning on this shared high-level strategy, allowing each agent to act with awareness of the strategies of the others. We model the joint action probability as the product of individual action probabilities conditioned on the consensus vector, which makes the overall framework amenable to optimization with single-agent Proximal Policy Optimization (PPO) \cite{schulman2017proximal}. This reformulation removes the sensitivity to action-generation order by construction and enables the whole system to be trained end-to-end as a single agent, rather than as $n$ interacting learners. We evaluate CMAT on a broad range of MARL benchmarks, including StarCraft II \cite{whiteson2019starcraft}, Multi-Agent MuJoCo \cite{de2020deep}, and Google Research Football \cite{kurach2020google}. Experimental results show that CMAT consistently outperforms strong baselines under fully observable cooperative settings.

\section{Preliminaries}
\subsection{Problem Formulation}
We consider cooperative Markov games, represented by $<\mathcal{N},\mathcal{O},\mathcal{A},\mathrm{R},\mathrm{P},\gamma>$ \cite{littman1994markov}. Here, $\mathcal{N} = \{1, 2, \ldots, n\}$ denotes the set of agents, $\mathcal{O} = \{o^1, o^2, \ldots, o^n\}$ denotes the observation set of each agent, and $\mathcal{A} = \{a^1, a^2, \ldots, a^n\}$ denotes the action set of the agents. The joint reward function is defined as $\mathrm{R}: \mathcal{O} \times \mathcal{A} \rightarrow \mathbb{R}$, and the transition function is given by $\mathrm{P}: \mathcal{O} \times \mathcal{A} \times \mathcal{O} \rightarrow \mathbb{R}$. The discount factor is denoted by $\gamma \in [0, 1)$. We further denote the joint policy by $\pi = \{\pi^1, \pi^2, \ldots, \pi^n\}$. At each time step $t$, all agents act simultaneously and receive a joint reward $\mathrm{R}(\mathcal{O}_t, \mathcal{A}_t)$, while the next observation $\mathcal{O}_{t+1}$ is generated according to $\mathrm{P}(\mathcal{O}_{t+1} \mid \mathcal{O}_{t}, \mathcal{A}_t)$. The objective is to maximize the long-term discounted cumulative reward, defined as $J = \sum_{t=0}^{\infty} \gamma^t \mathrm{R}(\mathcal{O}_t, \mathcal{A}_t)$, and the corresponding optimal policy is $\pi^* = \arg \max_{\pi} J$.

In this paper, we focus on a fully cooperative and fully observable setting, where the observations and policies of all agents are available to one another. This assumption is equivalent to a centralized controller with complete information \cite{wen2022multi,zhao2025triple}. Although this excludes settings that require decentralized execution or that involve partial observability and non-shared rewards, the resulting regime is far from a corner case: it directly captures a broad class of real-world systems in which a single operator observes the global state and dispatches all actors under a common objective, such as ride-hailing order dispatch, traffic signal control, warehouse robot coordination, and emergency fleet dispatch. In these applications, the value of a method is determined precisely by how well it coordinates actors under global information, which is the regime we target. In addition, we consider a model-free setting, in which the transition function $\mathrm{P}$ is learned implicitly rather than modeled explicitly. Under this formulation, we define the value functions for single agent and agent set in Appendix \ref{sec:definition}.

\paragraph{From Multi-Agent to Multi-Action Single-Agent.}
The setting above admits an equivalent single-agent description. Because all agents share the observation $\mathcal{O}$, share the reward $\mathrm{R}$, and act simultaneously, the collection of $n$ agents can be regarded as one meta-agent that, at each step, observes $\mathcal{O}$ and emits a single \emph{composite (multi-)action} $\mathcal{A}=(a^1,\ldots,a^n)$. This viewpoint is not merely a formal convenience: it is the natural model for the fully cooperative and fully observable regime we target. With a shared objective and complete information there is no informational asymmetry or conflict of interest to preserve across agents, so the per-agent boundaries carry no decision-theoretic content and can be absorbed into a single controller without loss. The reformulation yields a single-agent Markov Decision Process (MDP) $\langle\mathcal{O},\bar{\mathcal{A}},\mathrm{R},\mathrm{P},\gamma\rangle$, where the composite action space is the product $\bar{\mathcal{A}}=\mathcal{A}^1\times\cdots\times\mathcal{A}^n$, and $\mathrm{R}$, $\mathrm{P}$, $\gamma$ are inherited unchanged from the Markov game. A joint policy $\pi(\mathcal{A}\mid\mathcal{O})$ over the Markov game is exactly a single-agent policy over $\bar{\mathcal{A}}$, and the two problems share the same objective $J$. Casting the problem this way is what allows the entire system to be optimized with single-agent RL rather than with per-agent updates.

The remaining difficulty is that the composite action space grows as $\prod_i|\mathcal{A}^i|$, which is intractable to model as a flat categorical distribution. To keep the single-agent view while retaining tractability, we adopt a \emph{hierarchical} MDP in which the meta-agent first selects a high-level action, the latent consensus $c\in\mathcal{C}$, and then selects the low-level composite action conditioned on $c$. This induces the factorization
\begin{equation}
\pi(\mathcal{A}\mid\mathcal{O}) = \pi^c(c\mid\mathcal{O}) \prod_{i=1}^n \pi^i(a^i\mid\mathcal{O},c) ,
\label{eq:hier-mdp}
\end{equation}
where the high-level policy $\pi^c$ plays the role of a coordinator and the low-level factors $\pi^i$ produce individual actions in parallel. Two properties make this factorization a reasonable model rather than a crude approximation. First, factoring a high-dimensional composite action into per-component conditionals given a shared context is standard practice in single-agent policies with factored or multi-discrete action spaces \cite{schulman2017proximal,yu2022surprising}: the conditional independence is imposed only at the \emph{output layer}, while the network hidden layers are shared, so inter-action (inter-actor) correlations are still learned and carried through the shared representation, and only the final output distributions are assumed independent given that representation. Second, in our design the shared context is precisely the latent consensus $c$, which is refined to encode a coordinated team strategy before any action is emitted; conditioning every factor on $c$ therefore strengthens the coupling well beyond a generic shared trunk. Under this hierarchical MDP, CMAT does not decompose the problem into $n$ interacting learners; it instantiates a single multi-action policy whose internal structure captures inter-actor dependence through the shared latent $c$. We detail the corresponding architecture and optimization in the next section.

\subsection{Multi-Agent Transformer and Its Variants} \label{sec:MAT}

\begin{figure}[t]
    \centering
    \begin{tabular}{c c | c c}
        \toprule
        & & \multicolumn{2}{c}{\textbf{\textcolor{blue}{Agent 1}}} \\
        & & \textbf{\textcolor{blue}{A}} & \textbf{\textcolor{blue}{B}} \\
        \midrule
        \multirow{3}{*}{\rotatebox{90}{\textbf{\textcolor{red}{Agent 2}}}}
        & \rotatebox{90}{\textbf{\textcolor{red}{A}}} & 1 & -100 \\
        & & & \\
        & \rotatebox{90}{\textbf{\textcolor{red}{B}}} & 0 & 100 \\
        \bottomrule
    \end{tabular}
    \caption{Payoff Matrix of Cooperative Game: The joint actions $(A, A)$ and $(B, B)$ are both NE, while $(B, B)$ is the global optimum.}
    \label{tab:payoff_matrix}
\end{figure}

MAT \cite{wen2022multi} formulates the cooperative Markov game as a sequential model based on the Transformer architecture \cite{vaswani2017attention}, with the goal of capturing action dependencies among agents. Specifically, it first employs an encoder to extract observation features for all agents, denoted by $\hat{O} = \{\hat{o}^1, \hat{o}^2, \ldots, \hat{o}^n\}$, and predicts the V-value $\mathrm{V}(\hat{o}^i)$ from each observation feature $\hat{o}^i$. It then uses a decoder to generate agent actions {\em sequentially}, where self-attention captures inter-agent relationships and cross-attention models the dependence between actions and agents' observations. For simplicity, we assume that the decision order proceeds from agent 1 to agent $n$, and we denote the parameters of the critic and actor by $\phi$ and $\theta$, respectively. (\emph{Note that $\phi$ and $\theta$ partially overlap.}) The loss functions can then be expressed as:
\begin{equation}
\begin{aligned}
\mathrm{L}_{Critic}^{MAT}(\phi) & = \mathbf{E}_{i \in \mathcal{N}, t \in \mathcal{T}} \left[ (\mathrm{R}(\mathcal{O}_t, \mathcal{A}_t) + \gamma \mathrm{V}_{\phi^-}(\hat{o}^i_{t+1}) - \mathrm{V}_{\phi}(\hat{o}^i_{t}))^2 \right] \ , \\
\mathrm{L}_{Actor}^{MAT}(\theta) & = \mathbf{E}_{i \in \mathcal{N}, t \in \mathcal{T}} \left[ \min\left(r_t^{i}(\theta) \mathrm{A}(\mathcal{O}_t, \mathcal{A}_t),  \right. \right. \\
& \left. \left.  \mathrm{CLIP}(r_t^{i}(\theta), 1-\epsilon, 1+\epsilon) \mathrm{A}(\mathcal{O}_t, \mathcal{A}_t)\right) \right] \ , 
\label{eq:mat}
\end{aligned}
\end{equation}
where $\phi^-$ are the parameters of the target critic network, $r_t^{i}(\theta)$ is defined as $\frac{\pi_{\theta}^i(a_t^i|\mathcal{O}_t, a_t^{1:i-1})}{\pi_{\theta^-}^i(a_t^i|\mathcal{O}_t, a_t^{1:i-1})}$, $\theta^-$ are the network parameters used for sample collection, and $\mathcal{T}$ is defined as $\{0,1,2,\ldots,\infty\}$.  The advantage function is estimated using Generalized Advantage Estimation (GAE) \cite{schulman2015high} with the estimated V-value defined as $\mathrm{V}_{\phi}(\mathcal{O}_t) = \frac{1}{n} \sum_{i=1}^n \mathrm{V}_{\phi}(\hat{o}^i_{t})$.

Although MAT \cite{wen2022multi} has demonstrated promising performance when combined with Multi-Agent Advantage Decomposition \cite{kuba2022trust,zhong2024heterogeneous}, which theoretically guarantees that this optimization scheme can improve the joint advantage, several studies \cite{hu2025pmat,takayama2025aoad} have observed that the decision order of agents can substantially affect performance. A primary reason for this phenomenon lies in incorrect credit assignment induced by biased value-function estimation. In particular, for leading agents, the value estimation often fails to adequately capture the influence of subsequent agents. Given agent $i$, the real advantage function should be (derived from Eq. \ref{eq:value2} and Eq. \ref{eq:value3}):
\begin{equation}
\begin{aligned}
& \mathrm{A}^i(\mathcal{O}, a^{1:i-1},a^i)  = \mathrm{Q}^{1:i}(\mathcal{O}, a^{1:i}) - \mathrm{Q}^{1:i-1}(\mathcal{O}, a^{1:i-1}) \ , \\
& = \mathbf{E}_{\hat{a}^{i:n}} \left[ \mathrm{Q}(\mathcal{O}, [a^{1:i},\hat{a}^{i+1:n}]) - \mathrm{Q}(\mathcal{O}, [a^{1:i-1},\hat{a}^{i:n}])  \right] \ .
\label{eq:advantage}
\end{aligned}
\end{equation}
However, in the loss function defined in Eq. \ref{eq:mat}, $\mathrm{A}^i(\mathcal{O}, a^{1:i-1},a^i)$ is replaced by $\mathrm{A}(\mathcal{O}, \mathcal{A})$, thereby introducing the effects of subsequent agents' actions into the credit assignment of leading agents. Conversely, for following agents, another inconsistency arises between the joint V-value $\mathrm{V}(\mathcal{O})$, which depends only on joint observations, and the actors' behavior $\pi^{i}(\mathcal{O}, a^{1:i-1})$, which additionally conditions on the actions of preceding agents. Both issues can distort the optimization direction and ultimately hinder effective cooperation within the multi-agent system.


To illustrate a failure scenario involving MAT, consider the cooperative game depicted in Fig. \ref{tab:payoff_matrix}. This game consists of a single step where both agents have two possible actions. The values in the table represent the long-term team reward, with MAT following the decision order of Agent 1 leading to Agent 2. We assume that the initial V-value of all states is set to 0. If Agent 1 chooses action B and Agent 2 subsequently selects action A, this results in a substantial negative reward and advantage value of -100. As a consequence, the probabilities $ \pi^1(a^1=B|\mathcal{O}) $ and $ \pi^2(a^2=A|\mathcal{O}, a^1=B) $ will experience a significant decline. This situation can trigger an immediate reduction in action entropy, impeding the model's ability to explore the optimal action combination $(B, B)$ due to insufficient exploration-exploitation mechanisms; specifically, the probability of Agent 1 exploring action B, $ \pi^1(a^1=B|\mathcal{O}) $, decreases dramatically. Subsequently, when Agent 1 opts for action A with increased probability, the positive reward and advantage value will further elevate the action probability of $ \pi^1(a^1=A|\mathcal{O}) $ while correspondingly reducing $ \pi^1(a^1=B|\mathcal{O}) $, thereby exacerbating this dilemma.

Although \cite{hu2025pmat,takayama2025aoad} attempted to mitigate this issue by learning the decision order and agent actions simultaneously, such an approach substantially increases training complexity, as the search space over action orders grows to $n!$. More importantly, greater flexibility in ordering does not fundamentally resolve the optimization limitations of sequential decision making. {\em According to the theoretical analyses in \cite{kuba2022trust,liu2024maximum,zhong2024heterogeneous}, such sequential updating methods are guaranteed to converge only to a Nash Equilibrium (NE). For example, in the game shown in Fig. \ref{fig:example}, conventional MAT-based methods may converge to a Pareto-suboptimal NE. This limitation motivates us to move beyond sequential MARL formulations and instead reformulate cooperative MARL as a single-agent, multi-action problem, which allows the whole system to be trained end-to-end with single-agent RL and removes the order-dependent bias by construction.} To the best of our knowledge, such a formulation has not yet been established in the existing literature. We emphasize that this reformulation is a modeling and optimization-structure benefit rather than a claim of a stronger convergence guarantee: with deep function approximation, single-agent PPO itself carries no guarantee of reaching the global optimum, and we do not claim one for CMAT.

From this perspective, a closely related line of research is to formulate cooperative decision making directly within a SARL framework. Recently, \cite{zhao2025triple} proposed Triple-BERT, a SARL framework for ride-sharing tasks, which aims to directly learn the joint action probability using a BERT model \cite{devlin2019bert}. Specifically, it assumes that the joint action probability $\pi(\mathcal{A}|\mathcal{O})$ can be expressed as $\mathrm{z}(\prod_{i=1}^n \pi^i(a^i|\mathcal{O}))$, where $\mathrm{z}(\cdot)$ is an increasing mapping function. However, this assumption does not always hold in practice, and as a result, the method can still suffer from credit assignment issues similar to those in MAT. For instance, in the example shown in Fig. \ref{tab:payoff_matrix}, exploring the action $(B, A)$ can reduce the probability of $\pi^1(a^1=B|\mathcal{O})$, thereby making it less likely to explore the optimal action combination $(B, B)$. Due to page limitations, a more detailed discussion of related work is provided in Appendix \ref{sec:Related Work}.

\section{Methodology}
\subsection{Overview}

\begin{figure}[htbp]
\centering 
\includegraphics[width=0.45\textwidth]{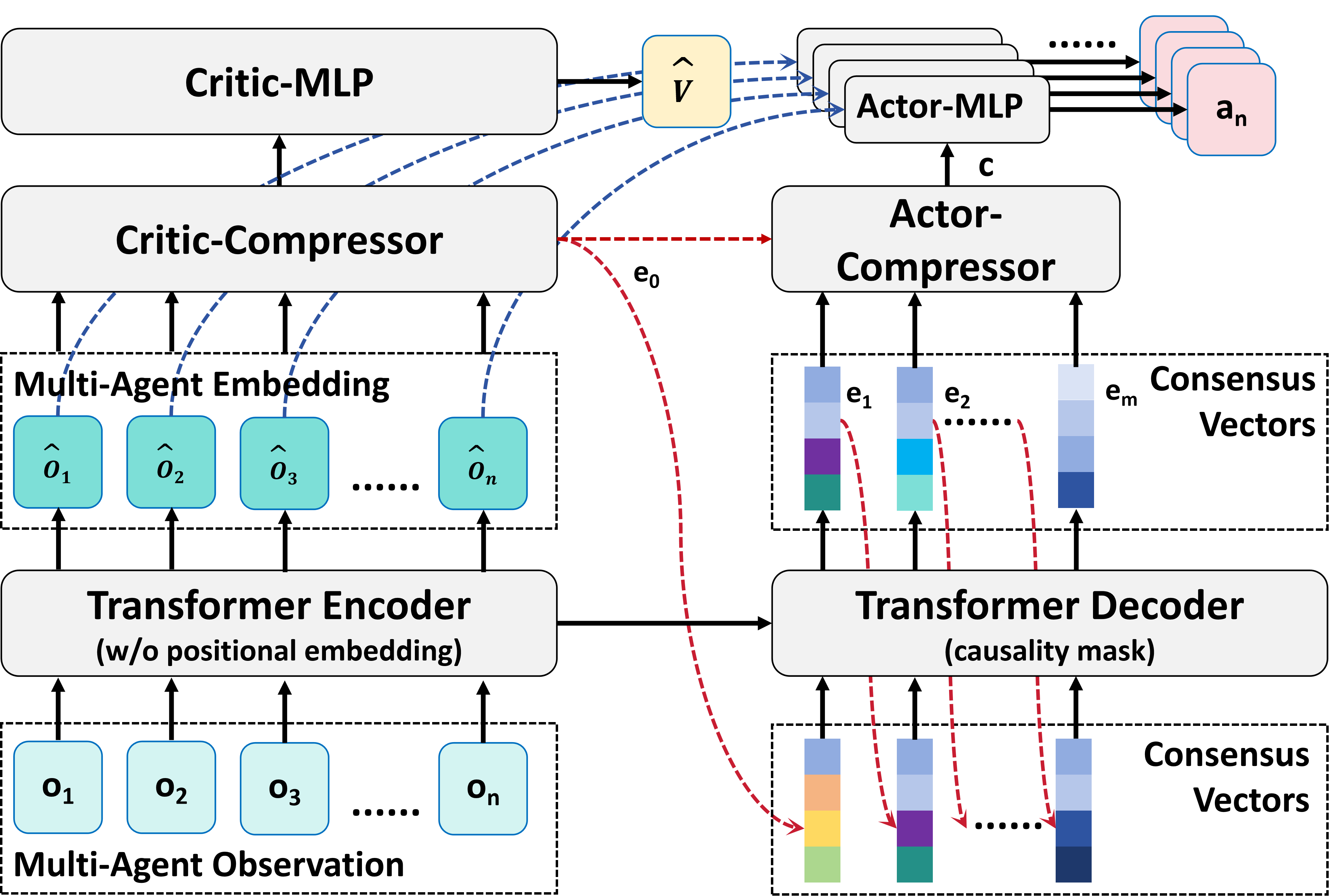}
\caption{Network Architecture: The Transformer encoder first extracts the features and relationships among the observations of all agents, compressing them into a single initial consensus vector. This vector is used for V-value estimation and iterated by the Transformer decoder to produce the final consensus vector. The final consensus vector is then combined with the extracted features of each agent's observation from the encoder to generate actions.}
\label{fig:main}
\end{figure}

In this paper, we propose an order-independent MAT \cite{wen2022multi} inspired by the consensus mechanism, named the CMAT. Unlike conventional MAT, our approach utilizes a decoder to iterate a consensus vector instead of specifying detailed actions, simulating the process by which all agents reach an agreement on their strategies in latent space. This instantiates the hierarchical MDP introduced in Eq. \ref{eq:hier-mdp}: we treat the whole system as a single multi-action agent, where the consensus $ c $ serves as a high-level action that summarizes the shared strategy and guides the low-level actions $ a^i $, and the per-agent action probability $ \pi^i(a^i|\mathcal{O},c) $ is modeled independently conditioned on $c$. In this way, we treat all agents as a unified entity and directly apply a single-agent PPO approach \cite{schulman2017proximal} to optimize the joint action policy $ \pi(\mathcal{A}|\mathcal{O}) $ of Eq. \ref{eq:hier-mdp}.

As discussed in Eq. \ref{eq:hier-mdp}, the conditional independence of the per-agent factors is imposed only at the output layer; because all action heads read from a common encoder and the shared consensus vector $c$, inter-actor correlations are captured within the shared hidden representation and summarized in $c$. CMAT therefore captures inter-actor dependence through shared representation and latent consensus, rather than through an explicit autoregressive coupling of the output distributions as in MAT.

The overall network architecture is illustrated in Fig. \ref{fig:main} and is built on the Transformer backbone \cite{vaswani2017attention}. First, the encoder extracts features and relationships among the observations of all agents using bi-directional self-attention, resulting in an observation embedding sequence $ \hat{\mathcal{O}} $. Next, a Critic-Compressor compresses this sequence into a single vector $ e^0 $, which is then used by the Critic-MLP to estimate the V-value $ \mathrm{V}(\mathcal{O}) $. This vector is also input to the Transformer decoder to perform consensus iteration for $ m $ times, producing a sequence of consensus vectors $\{e^1, e^2, \ldots, e^m\}$. Subsequently, an Actor-Compressor compresses the set $ \mathcal{E} = \{e^0, e^1, e^2, \ldots, e^m\} $ into a single vector $ c $, representing the final consensus and reflecting the potential strategy of all agents. Finally, the actor is generated by the Actor-MLP, which combines the consensus vector $ c $ with the observation embedding $ \hat{o}^i $ of each agent. Below we will discuss the details of each component within this architecture.

\subsection{Network Architecture}

\textbf{A. Encoder:} In the proposed framework, we first utilize a Transformer encoder to extract features and relationships among agents. Specifically, since the input needs to be order-independent, we eliminate the positional embedding that is commonly used in similar works \cite{zhao2025triple, wen2022multi}. The extraction process can be expressed as:
\begin{equation}
\begin{aligned}
\hat{\mathcal{O}} = \mathrm{Encoder}(\mathcal{O}) \in \mathbb{R}^{n,d} \ ,
\label{eq:encoder}
\end{aligned}
\end{equation}
where $d$ is the hidden dimension. Removing the positional embedding makes the encoder \emph{permutation equivariant}: for any permutation $\rho$ of the agent indices, $\mathrm{Encoder}(\rho\cdot\mathcal{O}) = \rho\cdot\mathrm{Encoder}(\mathcal{O})$, since self-attention without positional information treats its inputs as a set rather than a sequence. This property has been formally established for attention-based set models \cite{zaheer2017deep,lee2019set}. Consequently, permuting the order in which agent observations are fed to the encoder permutes the output embeddings identically but leaves their content unchanged; because the subsequent Critic-Compressor and Actor-Compressor aggregate these embeddings in a permutation-invariant manner (Eq. \ref{eq:critic_compressor}), the estimated V-value and the generated consensus are invariant to agent ordering, and each agent's action depends only on its own embedding $\hat{o}^i$ and the shared consensus $c$. CMAT therefore attains order independence by construction, without relying on any explicit ordering module.

Afterwards, we employ the Critic-Compressor to reduce $ \hat{\mathcal{O}} $ into a single vector, referred to as the initial consensus vector $ e^0 $, represented by the following process:
\begin{equation}
\begin{aligned}
x & \in \mathbb{R}^{N \times d_1} \ , \\
M & = \text{Softmax}(\text{MLP}(x), \text{dim=0}) \in \mathbb{R}_+^{N,h} \ , \\
z & = M^T \cdot x \in \mathbb{R}^{d_1 , h} \ , \\
y & = \text{MLP}(\text{Flatten}(z)) \in \mathbb{R}^{d_2} \ ,
\label{eq:critic_compressor}
\end{aligned}
\end{equation}
where $ x $ and $ y $ represent the input and output of the compressor, respectively, $ N $ is the length, $ d_1 $, $ h $, and $ d_2 $ are hidden dimensions, and $ M $ and $ z $ are intermediate variables. This structure is commonly used in sequence compression \cite{zhao2025csi,chou2024midibert}, motivated by the following considerations: First, the intermediate feature $ z $ is constructed from the original sequence $ x $ using $ h $ different combinations. For the $ i^{th} $ combination, it derives from a mixture of the original sequence, weighted by the $ i^{th} $ row of the weight matrix $ M $. Finally, the output hidden feature is produced by applying an MLP to compress the intermediate feature $ z $. In the context of the Critic-Compressor, $ x $ and $ y $ correspond to $ \hat{\mathcal{O}} $ and $ e^0 $, respectively, while dimensions $ d_1 $ and $ d_2 $ are both set to $ d $.

Subsequently, we use the Critic-MLP to process the initial consensus vector $e^0$ to obtain the estimated V-value, expressed as:
\begin{equation}
\begin{aligned}
\hat{\mathrm{V}}(\mathcal{O}) & = \text{Critic-MLP}(e^0) \ , \\
e^0 & = \text{Critic-Compressor}(\hat{\mathcal{O}}) \ .
\label{eq:critic}
\end{aligned}
\end{equation}

\textbf{B. Decoder:} For the decoder, we first auto-regressively iterate the consensus vector to achieve a converged strategy among agents, formulated as:
\begin{equation}
\begin{aligned}
\{e^1, e^2, \ldots, e^m\} = \mathrm{Decoder}(e^0) \ .
\label{eq:decoder}
\end{aligned}
\end{equation}
Unlike MAT, the positional embedding of the decoder is retained because we want the model to be aware of the convergence process of the consensus. 

Following this, we utilize the Actor-Compressor to compress $\mathcal{E}$ to obtain the final consensus vector, expressed as:
\begin{equation}
\begin{aligned}
c  = \text{Actor-Compressor}(\mathcal{E}) \ ,
\label{eq:actor_compressor}
\end{aligned}
\end{equation}
which follows the same procedural process as the Critic-Compressor defined in Eq. \ref{eq:critic_compressor}. \emph{Here, we choose to utilize the combination of the entire set $\mathcal{E}$ instead of only the last generated vector $e^m$ as the consensus. This decision helps to prevent any potential information loss that may occur during the iteration process.}

Finally, we combine the consensus vector $c$ with the observation feature $\hat{o}^i$ and feed them to the Actor-MLP to generate the action for agent $i$, given by:
\begin{equation}
\begin{aligned}
a^i  = \text{Actor-MLP}([\hat{o}^i;c]) \ .
\label{eq:actor_MLP}
\end{aligned}
\end{equation}

\subsection{Training Process}

\textbf{A. Training:} In our CMAT, we view all agents as a unified entity through a single-agent, multi-action perspective. The policy is optimized directly using a single-agent PPO approach:
\begin{equation}
\begin{aligned}
\mathrm{L}_{Critic}^{CMAT}(\phi)  = & \mathbf{E}_{t \in \mathcal{T}} \left[ (\mathrm{R}(\mathcal{O}_t, \mathcal{A}_t) + \gamma \mathrm{V}_{\phi^-}(\mathcal{O}_{t+1}) - \mathrm{V}_{\phi}(\mathcal{O}_t))^2 \right] \ , \\
\mathrm{L}_{Actor}^{CMAT}(\theta)  = & \mathbf{E}_{i \in \mathcal{N}, t \in \mathcal{T}} \left[ \min\left(\mathcal{R}_t(\theta) \mathrm{A}(\mathcal{O}_t, \mathcal{A}_t), \right. \right.  \\
&  \left. \left. \mathrm{CLIP}(\mathcal{R}_t(\theta), 1-\epsilon, 1+\epsilon) \mathrm{A}(\mathcal{O}_t, \mathcal{A}_t)\right) \right] \ ,
\label{eq:cmat}
\end{aligned}
\end{equation}
where the ratio $\mathcal{R}_t^{i}(\theta)$ is defined as (derived from Eq. \ref{eq:hier-mdp}):
\begin{equation}
\begin{aligned}
\mathcal{R}_t(\theta) & = \frac{\pi_{\theta}(\mathcal{A}_t|\mathcal{O}_t)}{\pi_{\theta^-}(\mathcal{A}_t|\mathcal{O}_t)} = \frac{\pi^c_{\theta}(c|\mathcal{O}) \prod_{i=1}^n \pi^i(a^i|\mathcal{O},c)}{\pi^c_{\theta^-}(c^-|\mathcal{O}) \prod_{i=1}^n \pi^i(a^i|\mathcal{O},c^-)}  \\
& = \frac{\prod_{i=1}^n \pi^i_\theta(a^i|\mathcal{O},\mu_\theta(\mathcal{O}))}{\prod_{i=1}^n \pi^i_{\theta^-}(a^i|\mathcal{O},\mu_{\theta^-}(\mathcal{O}))} \ ,
\label{eq:ratio}
\end{aligned}
\end{equation}
where $c=\mu_{\theta}(\mathcal{O})$ and $c^-=\mu_{\theta^-}(\mathcal{O})$ are consensus vectors generated by the current and old actors, respectively, and $\mu_\theta(\cdot)$ represents the consensus generation policy. Since the consensus generation process is deterministic, similar to TD3 and DDPG, both $ \pi^c_{\theta}(c|\mathcal{O}) $ and $ \pi^c_{\theta^-}(c^-|\mathcal{O}) $ are fixed at 1. The consensus generation policy is updated implicitly that during backpropagation, the gradient flows back to $ \mu_\theta(\mathcal{O}) $. We adopt a deterministic consensus for training stability and to avoid importance weighting over the latent space; a stochastic consensus policy, together with the associated cooperative Stackelberg (bilevel) formulation, is a natural generalization that we discuss as future work in Appendix \ref{app:theory-future}.

To illustrate how CMAT resolves the dilemma shown in Fig. \ref{tab:payoff_matrix}, consider again the joint action $(B, A)$. Under our formulation, the affected probability is $\pi^1(a^1=B|\mathcal{O},c)$, which is conditioned on the specific consensus $c$ and therefore tied only to the particular latent strategy that led to the suboptimal joint action. As a result, reducing this probability mainly indicates that the consensus $c$ (i.e., $\mu(\mathcal{O})$) is suboptimal, rather than penalizing action $B$ for Agent 1 in an unconditional manner. In contrast, the policy under the optimal consensus $c^*$, namely $\pi^1(a^1=B|\mathcal{O},c^*)$, remains unaffected. Meanwhile, the consensus generation module $\mu(\mathcal{O})$ is updated through gradient descent based on the overall advantage, gradually steering $\mu(\mathcal{O})$ toward $c^*$ over time.


\textbf{B. Fine-tuning:} In the joint training phase above, the consensus generation policy $\mu_{\theta}(\mathcal{O})$ and the action policy $\pi_{\theta}(\cdot|\mathcal{O},c)$ are updated at the same time. This creates a moving-target problem: the action heads are chasing a consensus representation that is itself still changing, and conversely the consensus is optimized against action heads that have not yet adapted to it. Such coupled updates can interfere with each other and slow convergence, a difficulty familiar from bilevel and actor-critic style optimization where the inner and outer variables are trained jointly. To decouple the two, we introduce a fine-tuning phase that fixes one part of the model and refines the other, using the same loss as in Eq. \ref{eq:cmat}. Concretely, we provide two complementary variants:

\textbf{(i) Consensus Enhancement:} We fine-tune the critic $\text{V}_{\phi}(\mathcal{O})$ and the consensus generation policy $\mu_{\theta}(\mathcal{O})$ while keeping the action policy $\pi_{\theta}(\cdot|\mathcal{O},c)$ fixed, so that only the gradients of the Critic-MLP, decoder, and Actor-Compressor are active. This variant treats the current action heads as a fixed set of followers and searches for a better high-level coordination signal for them. It is the natural choice when the per-agent action heads are already competent but the team-level coordination encoded in $c$ is suboptimal, i.e., when the bottleneck is coordination rather than individual control.

\textbf{(ii) Action Policy Enhancement:} We fine-tune the critic $\text{V}_{\phi}(\mathcal{O})$ and the action policy $\pi_{\theta}(\cdot|\mathcal{O},c)$ while fixing the consensus generation policy $\mu_{\theta}(\mathcal{O})$, so that only the Critic-MLP and Actor-MLP receive gradients. This variant freezes the coordination signal and lets each agent better exploit an already-informative consensus. It is preferable when the consensus is already meaningful but the individual action heads underfit it, i.e., when the bottleneck is per-agent execution rather than coordination.

Because the two variants optimize the same objective while alternating which block is held fixed, they can be viewed as two directions of a single block-wise refinement of the joint policy from a single-agent perspective; the appropriate choice for a given task depends on which of the two bottlenecks dominates. Empirically we find that both yield similar final performance, indicating that CMAT is not sensitive to this choice in the benchmarks we consider. The whole training process is provided at Appendix \ref{sec:algorithm}. A rigorous analysis of why the deployed deterministic architecture is effective is given in Appendix \ref{app:theory}, and a higher-level discussion of a more general stochastic-consensus framework as future work is provided in Appendix \ref{app:theory-future}.

\section{Experiment} \label{sec:experiment}

\subsection{Experiment Setup}

To validate the efficiency of our proposed method, we evaluate its performance in a series of benchmark MARL experiment scenarios, including:
\begin{itemize}[left=0pt]
    \item \textbf{StarCraft II \cite{whiteson2019starcraft}:} A challenging real-time strategy game environment that provides complex micromanagement tasks for testing multi-agent cooperation and coordination. Since many advanced MARL methods can achieve a 100\% win rate in simpler environments \cite{wen2022multi}, we focus on the most difficult task scenarios for comparison in this paper, including ``MMM2", ``6h vs 8z", and ``3s5z vs 3s6z".
    
    \vspace{0.1cm}
    \item \textbf{Multi-Agent MuJoCo \cite{de2020deep}:} A set of continuous control robotic tasks adapted from MuJoCo, where multiple agents must coordinate to control a single or multiple articulated bodies, testing fine-grained cooperation. We select three challenging scenarios, each with a single agent controlling one body: ``8$\times$1-Agent Ant", ``6$\times$1-Agent HalfCheetah", and ``6$\times$1-Agent Walker2d".
    
    \vspace{0.1cm}
    \item \textbf{Google Research Football \cite{kurach2020google}:} A highly realistic football simulation platform that requires agents to master individual skills and long-horizon teamwork strategies in a dynamic, physics-based environment. The detailed tasks include ``academy counterattack easy", ``academy pass and shoot with keeper", and ``academy 3 vs 1 with keeper".
\end{itemize}

To illustrate the superior performance of our method, we compare it against several strong baselines, using the same settings as reported in \cite{wen2022multi}:
\begin{itemize}[left=0pt]
    \item \textbf{MAT \cite{wen2022multi}:} The first centralized observation with a sequential decision-making MARL framework. More details are presented in Section \ref{sec:MAT}.

    
    \vspace{0.1cm}
    \item \textbf{PMAT \cite{hu2025pmat}:} Based on MAT, PMAT introduces an additional module to determine agent action order, grounded in the theory of Plackett-Luce sampling \cite{luce1959individual,plackett1975analysis}. AOAD-MAT \cite{takayama2025aoad}, a synchronous work, employs a similar method to PMAT but decides action order followed by detailed action decision, while AOAD-MAT does both simultaneously. Given the high similarity between them and the fact that only PMAT provides official code, we select PMAT as the benchmark here.

    \vspace{0.1cm}
    \item \textbf{Triple-BERT \cite{zhao2025triple}:} Triple-BERT is the first centralized SARL framework for the ride-sharing task, utilizing BERT for observation feature and relationship extraction and processing a large action space through an action decomposition mechanism. 
    
    \vspace{0.1cm}
    \item \textbf{HAPPO \cite{kuba2022trust,zhong2024heterogeneous}:} HAPPO first illustrates and proves the efficiency of sequential optimization among agents in cooperative MARL scenarios, providing a strong foundation for MAT. Based on \cite{zhong2024heterogeneous}, HAPPO serves as the SOTA method in the family of Heterogeneous-Agent Reinforcement Learning (HARL) methods. 
    
    \vspace{0.1cm}
    \item \textbf{MAPPO \cite{yu2022surprising}:} MAPPO is a popular and strong CTDE baseline in cooperative MARL, utilizing a centralized critic during training to provide direction for actors with global information, while actors operate independently during execution.
\end{itemize}


More details about the experiment configurations can be found at Appendix \ref{sec:configuration}.

\subsection{Experiment Results}

\begin{figure*}[t!]
\centering 
\subfloat[Legend]{\includegraphics[width=0.8\textwidth]{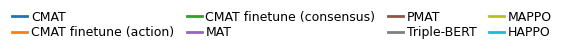}} \\
\subfloat[MMM2]{\includegraphics[width=0.33\textwidth]{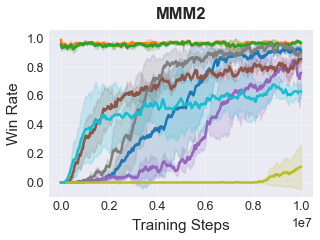}} 
\subfloat[6h vs 8z]{\includegraphics[width=0.33\textwidth]{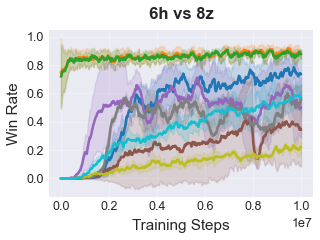}}
\subfloat[3s5z vs 3s6z]{\includegraphics[width=0.33\textwidth]{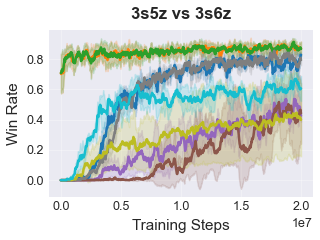}} \\
\subfloat[8$\times$1-Agent Ant]{\includegraphics[width=0.33\textwidth]{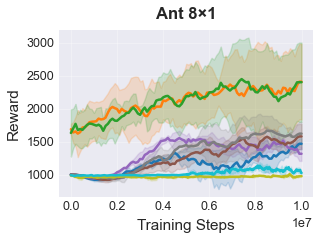}} 
\subfloat[6$\times$1-Agent HalfCheetah]{\includegraphics[width=0.33\textwidth]{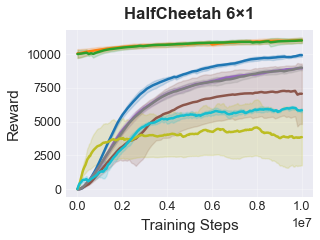}}
\subfloat[6$\times$1-Agent Walker2d]{\includegraphics[width=0.33\textwidth]{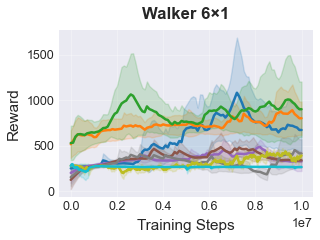}} \\
\subfloat[academy counterattack easy]{\includegraphics[width=0.33\textwidth]{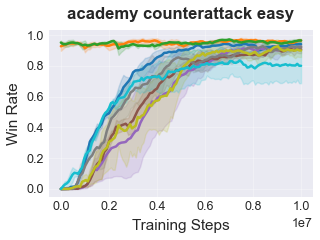}} 
\subfloat[academy pass and shoot with keeper]{\includegraphics[width=0.33\textwidth]{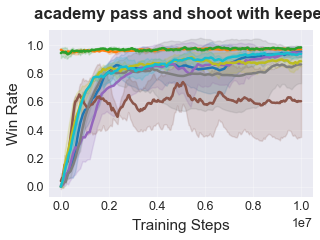}}
\subfloat[academy 3 vs 1 with keeper]{\includegraphics[width=0.33\textwidth]{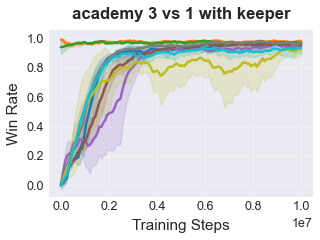}} \\
\caption{Training Curves under 5 Random Seeds: The shadow parts represent the standard deviation.}
\label{fig:exp}
\end{figure*}

The experimental results are shown in Fig. \ref{fig:exp}. Note that the two CMAT-finetune curves start with high performance because they are initialized from the well-trained CMAT. We observe that CMAT achieves superior performance in most scenarios, and its advantage becomes more evident after fine-tuning, with both CMAT-finetune variants achieving the best performance across all scenarios. The improvement is consistent across the three benchmark families, which stress different aspects of coordination: discrete micromanagement in SMAC, continuous high-dimensional control in MA-MuJoCo, and sparse-reward long-horizon cooperation in Google Research Football. This breadth indicates that the gains are not tied to a particular action type or reward structure, but follow from the consensus formulation itself. The margin over the baselines tends to widen on the harder, larger-agent scenarios, where coordinating actions is most difficult and the shared consensus provides the most benefit; on the easier scenarios the strong baselines are already close to saturation, so the gap is naturally smaller. Beyond mean return, CMAT typically exhibits comparable or lower variance across the five seeds than the MAT-family baselines, which we attribute to the cleaner single-agent optimization structure that avoids the sequential per-agent updates. Notably, the results of CMAT-finetune (action) clearly demonstrate the effectiveness of our consensus mechanism. If the consensus were invalid and independent of the states, it could be disregarded by the action head, in which case CMAT would degrade to Triple-BERT, where all agents take actions simultaneously without any consensus.


\subsection{Ablation Study and Sensitivity Analysis}

We further conduct ablation studies and hyper-parameter sensitivity analyses to isolate the contribution of each design choice; the full setup, figures, and per-scenario curves are deferred to Appendix \ref{sec:ablation}. Two conclusions are validated. First, replacing the mixed consensus produced by our Actor-Compressor with the single last consensus vector $e^m$ degrades performance in most scenarios, confirming that aggregating the whole consensus trajectory $\mathcal{E}$ preserves information that would otherwise be lost during auto-regressive decoding. Second, setting the number of decoder iterations to the agent count $m=n$ is a robust default: fewer iterations under-form the consensus, while more inject redundant information that complicates training. These results indicate that the observed gains stem from the consensus mechanism itself rather than from incidental hyper-parameter tuning.


\section{Conclusion}
In this paper, we proposed CMAT, a novel centralized method for fully observable cooperative MARL tasks. By using a Transformer decoder to iteratively generate a consensus representation, CMAT reformulates cooperative MARL as a hierarchical single-agent, multi-action problem, in which all agents act simultaneously while conditioning on a shared consensus and full information about one another. This formulation lets the whole system be optimized with single-agent PPO and alleviates several structural limitations of conventional MAT, including order dependency and actor-critic inconsistency, and it avoids the sequential updating scheme under which MAT-family methods are guaranteed to converge only to a Nash Equilibrium. We do not claim a stronger convergence guarantee for the deep implementation; rather, the benefit is a cleaner optimization structure, which is reflected in the empirical results. Extensive experiments on a series of standard MARL benchmarks demonstrate that CMAT consistently outperforms strong baselines. Additional discussions are provided in Appendix \ref{sec:Discussions}.
\clearpage
\putbib
\end{bibunit}


\clearpage
\appendix

\begin{bibunit}



\begin{figure*}[t!]
\centering 
\includegraphics[width=0.9\textwidth]{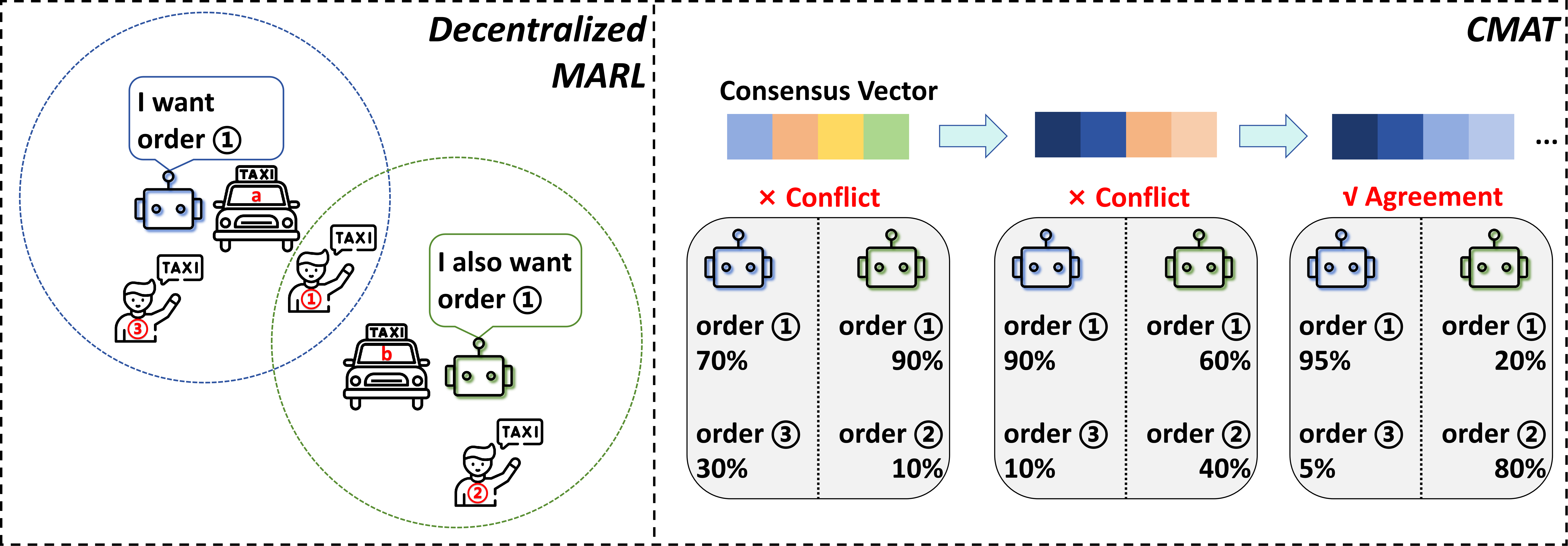}
\caption{Comparison between CMAT and Conventional Decentralized MARL Methods.}
\label{fig:example}
\end{figure*}

\section{Related Work} \label{sec:Related Work}
\subsection{Cooperative Multi-Agent Reinforcement Learning}

Cooperative MARL has found extensive applications across diverse domains, including power control, robotic fleet management, and ride-hailing systems \cite{yuan2023survey}. As highlighted in a comprehensive survey \cite{oroojlooy2023review}, existing methods can be taxonomically classified into five major categories: independent learning, centralized critic, value decomposition, consensus-based, and communication-based approaches.

Early research predominantly focused on independent learners, which represent the most straightforward adaptation of single-agent RL to multi-agent settings. By treating each agent as an independent entity and considering others as part of the environment, standard RL algorithms can be readily applied. However, this paradigm suffers from fundamental limitations: the environment becomes inherently non-stationary due to concurrently learning peers, and agents tend to converge to local optima by maximizing individual rewards while neglecting global cooperation. Although subsequent improvements such as Hysteretic Q-Learning \cite{matignon2007hysteretic} were proposed to mitigate these issues, independent learning still struggles in large-scale systems and sparse-reward scenarios.

To address these challenges, the Centralized Training with Decentralized Execution (CTDE) and Centralized Training with Centralized Execution (CTCE) paradigms have been widely adopted in centralized critic and value decomposition methods \cite{oroojlooy2023review}. In centralized critic approaches, actor-critic algorithms such as PPO, SAC, and DDPG are extended by replacing the critic with a centralized counterpart that observes global information during training, thereby stabilizing the learning process while preserving decentralized execution. To further enhance representational capacity, attention mechanisms inspired by the Transformer architecture \cite{vaswani2017attention} have been incorporated to model inter-agent relationships \cite{mao2018modelling}. Foerster et al. \cite{foerster2018counterfactual} proposed Counterfactual Multi-Agent (COMA) policy gradients to resolve the credit assignment problem through counterfactual baselines, an idea subsequently refined in works such as \cite{li2021shapley}. Nevertheless, many of these methods encounter the Curse of Dimensionality (CoD) as the number of agents scales up.

Value decomposition methods, by contrast, focus on factorizing the global reward into individual credit assignments, enabling agents to optimize collective objectives rather than selfish returns. The seminal Value Decomposition (VD) network \cite{sunehag2018value} pioneered this direction but suffered from the ``lazy agent" problem due to its simplistic additive Q-value factorization. Subsequent advances have significantly improved representational capacity through more expressive mixing architectures, including QMIX \cite{rashid2020monotonic} which enforces monotonicity constraints, Value-Decomposition Actor-Critic (VDAC) \cite{su2021value}, and QTRAN \cite{son2019qtran} which lifts monotonicity restrictions. However, most of these approaches require a mixing network to map individual Q-values to the global Q-value, and even with hyper-networks for parameter generation, they still grapple with scalability as agent populations grow \cite{hao2022hierarchical}.

Consensus and communication methods emerged later as a means to balance cooperation efficiency against the CoD challenge. In these paradigms, agents exchange information only with neighbors or selected peers rather than broadcasting globally. Consensus-based approaches leverage sparse communication to achieve policy alignment among agents, often with convergence guarantees under linear function approximation \cite{varshavskaya2009efficient}. However, many such methods require multiple communication rounds, rendering them impractical for real-time applications like ride-sharing where low latency is paramount. Communication-based methods instead focus on designing efficient mechanisms for determining what information to share and with whom. CommNet \cite{sukhbaatar2016learning} pioneered this direction by broadcasting each agent's hidden features derived from local observations. Yet, similar to Mean Field MARL approaches \cite{yang2018mean}, CommNet considers only averaged influences, overlooking fine-grained inter-agent relationships. Subsequent attention-based methods were introduced to weigh the importance of information from different sources \cite{das2019tarmac, li2024context}. However, communication-based methods face nontrivial training difficulties, particularly in early stages when communicated messages carry limited meaningful information. Furthermore, they often involve inherent trade-offs between cooperation efficacy, communication overhead, and message content—issues intimately tied to CoD.

\subsection{Transformer in Reinforcement Learning}
Inspired by the success of the Transformer \cite{vaswani2017attention} in Large Language Models (LLMs), particularly its scalability, strong generalization, and ability to capture long-range dependencies, the architecture has been widely adopted in fields such as computer vision \cite{dosovitskiy2021an}, signal processing \cite{zhao2026automatic}, and reinforcement learning  \cite{wen2022multi}. According to \cite{hu2024transforming}, Transformer-based methods in RL can be broadly categorized into three main applications: (i) architecture enhancement, where Transformers serve as more powerful backbones to improve policy or model capacity; (ii) offline RL, where they learn from sequential trajectory data; and (iii) online RL, where they are integrated to enrich the learning paradigm.

In the context of architecture enhancement, most works leverage the Transformer's ability to model long-term temporal dependencies, particularly in Partially Observable Markov Decision Processes (POMDPs), where solutions previously relied on recurrent architectures such as DRQN \cite{hausknecht2015deep}. For example, the Gated Transformer-XL (GTrXL) \cite{parisotto2020stabilizing} introduces a gating mechanism with pathway skip connections, achieving improved feature extraction from historical trajectories. Another line of research employs Transformers for environment modeling in model-based RL, capitalizing on their strong sequence prediction capabilities. TransDreamer \cite{chen2021transdreamer}, for instance, integrates Transformers into the Dreamer framework to construct a stochastic world model, outperforming conventional RNN-based counterparts \cite{hafnerdream}.

For offline RL, the most prominent methods are the Trajectory Transformer (TT) \cite{janner2021offline} and the Decision Transformer (DT) \cite{chen2021decision}. TT models each feature of state, action, and reward as separate tokens and formulates behavior cloning as a next-token prediction task. In contrast, DT treats each state-action-reward triple as a single entity and frames the problem as a reward-to-go (RTG) guided sequence prediction task. These two works have profoundly influenced subsequent research. For instance, Online DT (ODT) \cite{zheng2022online} extends DT to an offline pre-training and online fine-tuning paradigm, addressing the distribution shift between offline trajectories and online interactions. Q-Learning DT (QDT) \cite{yamagata2023q} further refines dataset quality by relabeling RTGs.

Recently, the success of LLMs has prompted researchers to explore how their generalization properties, such as few-shot learning and rich representations, can benefit RL. Building on transfer learning theory \cite{pan2009survey}, several studies directly leverage pre-trained Transformers to initialize RL policies \cite{reid2022can, li2022pre}. Others introduce masked prediction tasks during offline RL pre-training to enhance feature extraction \cite{liu2022masked, wu2023masked}. \cite{zhao2025triple} proposes a multi-agent pre-training method to improve single-agent RL under data scarcity. Inspired by the in-context learning capabilities of LLMs, \cite{xu2022prompting, lin2022contextual} investigate prompting-based methods that guide policies by conditioning on expert demonstrations. Moreover, \cite{lee2022multi, reed2022generalist} train policies across multiple tasks in a multi-task setting, achieving generalization through direct exposure to diverse environments.

The aforementioned methods predominantly operate in offline or offline-to-online settings, as Transformers in single-agent RL are primarily used to model historical trajectories. In contrast, in MARL, Transformers are well-suited for capturing inter-agent relationships, leading to a surge of online MARL methods. The mat (MAT) \cite{wen2022multi} pioneered this direction by employing an encoder to model agent interactions and a decoder for sequential decision-making. Subsequent works have built upon MAT in various ways. For example, \cite{takayama2025aoad, hu2025pmat} jointly optimize action selection and decision order, highlighting the significant impact of ordering on performance. CommFormer \cite{hu2024learning} integrates graph attention networks to enable communication-efficient decentralized execution in a centralized training paradigm. Other approaches, such as MaskMA \cite{liu2024maskma} and UPDeT \cite{hu2021updet}, focus on learning general action representations with Transformers for agent interaction modeling. However, they still require task-specific architectural design, which limits their practical applicability.

\section{Value Function Definition} \label{sec:definition}
The standard value functions are defined as follows:
\begin{equation}
\begin{aligned}
\mathrm{V}_\pi(\mathcal{O}_t) & = \mathbf{E}_\pi \left[ \sum_{\tau=t}^{\infty} \gamma^{\tau-t} \mathrm{R}(\mathcal{O}_\tau, \mathcal{A}_\tau) \mid \mathcal{O}_t \right]  \ , \\
\mathrm{Q}_\pi(\mathcal{O}_t,\mathcal{A}_t) & = \mathbf{E}_\pi \left[ \sum_{\tau=t}^{\infty} \gamma^{\tau-t} \mathrm{R}(\mathcal{O}_\tau, \mathcal{A}_\tau) \mid \mathcal{O}_t, \mathcal{A}_t \right]  \ , \\
\mathrm{A}_\pi(\mathcal{O}_t, \mathcal{A}_t) & = \mathrm{Q}_\pi(\mathcal{O}_t, \mathcal{A}_t) - \mathrm{V}_\pi(\mathcal{O}_t)  \ ,
\label{eq:value}
\end{aligned}
\end{equation}
where $\mathrm{V}_\pi(\cdot)$, $\mathrm{Q}_\pi(\cdot, \cdot)$, and $\mathrm{A}_\pi(\cdot, \cdot)$ denote the observation value function, the observation-action value function, and the advantage function, respectively.

We further define the Q-value function for a specific agent set $\psi$ as \cite{wen2022multi,kuba2022trust,zhong2024heterogeneous}:
\begin{equation}
\begin{aligned}
\mathrm{Q}^{\psi}_{\pi}(\mathcal{O}_t, a^{\psi}_t) = \mathbf{E}_{\hat{a}^{-\psi}_t \sim \pi} \left[ \mathrm{Q}_{\pi}(\mathcal{O}_t, [a^{\psi}_t, \hat{a}^{-\psi}_t]) \right] \ , \\
\label{eq:value2}
\end{aligned}
\end{equation}
where $-\psi$ denotes the complement of $\psi$. Based on this definition, the advantage function for $\psi$ given the actions of $\Psi$ is defined as
\begin{equation}
\begin{aligned}
\mathrm{A}^{\psi}_{\pi}(\mathcal{O}_t, a^{\Psi}_t, a^{\psi}_t) = \mathrm{Q}^{\psi \cup \Psi}_{\pi}(\mathcal{O}_t, [a^{\psi}_t, a^{\Psi}_t]) - \mathrm{Q}^{\psi}_{\pi}(\mathcal{O}_t, a^{\psi}_t) \ , 
\label{eq:value3}
\end{aligned}
\end{equation}
where $\psi$ and $\Psi$ are disjoint sets. For notational simplicity, we omit the policy symbol $\pi$ in the following sections whenever no ambiguity arises.

\section{Algorithm} \label{sec:algorithm}
The detailed training procedure of the proposed CMAT is presented in Algorithm~\ref{alg:cmat}. Furthermore, we illustrate the conceptual distinction between our method and conventional decentralized MARL approaches in Fig.~\ref{fig:example}, which highlights the key motivation behind our design.

\begin{algorithm}[htbp]
\caption{Consensus Multi-Agent Transformer (CMAT)}
\label{alg:cmat}
\begin{algorithmic}[1]
\Require Number of agents $n$, consensus iterations $m$, PPO hyper-parameters $\epsilon$, $\gamma$, GAE parameter $\lambda$, total training steps $T_{\text{total}}$
\Ensure Trained policy network $\theta$ and critic network $\phi$

\State Initialize network parameters $\theta$, $\phi$, and target critic $\phi^- \gets \phi$

\State \textbf{Procedure ActionSelection}$(\mathcal{O})$ \Comment{Used for experience collection}
\State \quad $\hat{\mathcal{O}} \gets \mathrm{Encoder}(\mathcal{O})$ \Comment{order-independent}
\State \quad $e^0 \gets \mathrm{CriticCompressor}(\hat{\mathcal{O}})$
\State \quad $\hat{V}(\mathcal{O}) \gets \mathrm{CriticMLP}(e^0)$
\State \quad $\mathcal{E} \gets \{e^0\}$
\For{$k = 1$ \textbf{to} $m$}
    \State \quad $e^k \gets \mathrm{Decoder}(e^{k-1}, k)$ \Comment{auto-regressive with positional index $k$}
    \State \quad $\mathcal{E} \gets \mathcal{E} \cup \{e^k\}$
\EndFor
\State \quad $c \gets \mathrm{ActorCompressor}(\mathcal{E})$
\For{$i = 1$ \textbf{to} $n$}
    \State \quad $a^i \sim \pi^i_\theta(\cdot \mid \hat{o}^i, c)$ \Comment{ActorMLP$([\hat{o}^i; c])$}
\EndFor
\State \quad \textbf{return} $\mathcal{A} = \{a^1, \dots, a^n\}$, $\hat{V}(\mathcal{O})$

\Statex
\State \textbf{A.Training}
\While{training steps $< T_{\text{total}}$}
    \State Collect trajectory $\tau = \{(\mathcal{O}_t, \mathcal{A}_t, R_t, \mathcal{O}_{t+1})\}$ by calling \textbf{ActionSelection}$(\mathcal{O}_t)$
    \State Compute advantages $\hat{A}_t$ using GAE with $\hat{V}_{\phi^-}$ and $\lambda$
    \State Compute critic loss:
    \State \qquad $\mathcal{L}_\mathrm{critic}(\phi) \gets \mathbb{E}_t \left[ (R_t + \gamma \hat{V}_{\phi^-}(\mathcal{O}_{t+1}) - \hat{V}_\phi(\mathcal{O}_t))^2 \right]$
    \State Compute importance ratio:
    \State \qquad $R_t^i(\theta) \gets \dfrac{\prod_{j=1}^n \pi^j_\theta(a^j_t \mid \hat{o}^j_t, c_\theta)}{\prod_{j=1}^n \pi^j_{\theta^-}(a^j_t \mid \hat{o}^j_t, c_{\theta^-})}$
    \State Compute actor loss:
    \State \qquad $\mathcal{L}_\mathrm{actor}(\theta) \gets \mathbb{E}_{i,t} \left[ \min\!\left(\mathcal{R}_t^i(\theta)\hat{A}_t, \text{CLIP}(\mathcal{R}_t^i(\theta), 1-\epsilon, 1+\epsilon)\hat{A}_t\right) \right]$
    \State Update $\theta$ and $\phi$ by minimizing $\mathcal{L}_\mathrm{actor} + \mathcal{L}_\mathrm{critic}$
    \State Soft-update target network: $\phi^- \gets \tau \phi + (1-\tau) \phi^-$
\EndWhile

\Statex
\State \textbf{B. Fine-tuning} \Comment{Continue from current $\theta$, $\phi$}
\If{Consensus Enhancement}
    \State Freeze all Actor-MLP layers
    \State Continue training by updating only Critic-MLP, Decoder, and Actor-Compressor
\ElsIf{Action Policy Enhancement}
    \State Freeze Encoder, Decoder, Critic-Compressor, and Actor-Compressor
    \State Continue training by updating only Critic-MLP and all Actor-MLP layers
\EndIf

\State \textbf{return} Trained policy network $\theta$ and critic network $\phi$
\end{algorithmic}
\end{algorithm}

\section{Experiment Supplement}
\subsection{Experiment Configurations} \label{sec:configuration}

For all evaluated scenarios, the model configurations follow those established in MAT \cite{wen2022multi}, with the exception of setting the rollout threads to 8 due to hardware limitations. These configurations are also adopted for comparative methods unless their original papers or official repositories specify alternative setups. Our implementation is built upon the official MAT repository, available at \url{https://github.com/PKU-MARL/Multi-Agent-Transformer}.

It is worth noting that certain experimental results, particularly those on StarCraft II, may deviate from the originally reported figures in MAT. These discrepancies arise from differences in random seeds and software versions. To ensure reproducibility, we detail the exact software versions used in our experiments. These version specifications apply uniformly to all comparative methods:
\begin{itemize}[left=0pt]
\item \textbf{StarCraft II}: PySC2 version 4.0.0, SMAC version 1.0.0, with underlying StarCraft II game version Base55958.
\vspace{0.1cm}
\item \textbf{Multi-Agent MuJoCo}: MuJoCo version 3.4.0, PettingZoo version 1.25.0. The evaluated robotic environments include HalfCheetah-v2, Ant-v2, and Walker2D-v2.
\vspace{0.1cm}
\item \textbf{Google Research Football}: GFootball version 2.10.2.
\end{itemize}

All models were trained using the PyTorch framework \cite{paszke2019pytorch} on a workstation running Windows 11, equipped with an Intel(R) Core(TM) i7-14700KF processor and an NVIDIA RTX 4080 graphics card. The GPU occupation is around 0.5-1.5 GB during the whole training process.

\emph{Here we want to emphasis due to hardware limitations, we were unable to run the original simulation versions with the recommended number of rollout threads specified in MAT \cite{wen2022multi}, as these settings frequently caused system instability on our device. While we acknowledge that this adjustment may prevent the evaluated methods from achieving their theoretical performance upper bounds, the comparison remains fair, as all methods were tested under identical modified conditions.}

\subsection{Ablation Study and Sensitivity Analysis} \label{sec:ablation}

To further illustrate the efficacy of each proposed module, we conduct a series of ablation studies and sensitivity analyses of hyper-parameters, including:
\begin{itemize}[left=0pt]
    \item \textbf{Consensus Mixture Versus Last Consensus:} We propose the decoder (actor) compressor to mix the middle-generated consensus $\mathcal{E}$ to avoid information loss during the auto-regressive decoding process. Here, we compare the performance of our mixture method against the direct use of the last generated consensus vector $e^m$.

    \vspace{0.1cm}
    \item \textbf{Impact of Consensus Iteration Times:} In our method, there is only one newly introduced hyperparameter compared to CMAT, namely the decoder iteration times $m$. We default this to the number of agents $n$, with the intuition that in the worst-case scenario, CMAT can degrade to MAT. Here, we additionally compare the performance under $m=0, \lfloor \frac{n}{2} \rfloor , 2n$.
\end{itemize}

\begin{figure*}[t!]
\centering
\subfloat[Legend]{\includegraphics[width=0.8\textwidth]{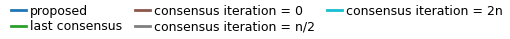}} \\
\subfloat[MMM2]{\includegraphics[width=0.33\textwidth]{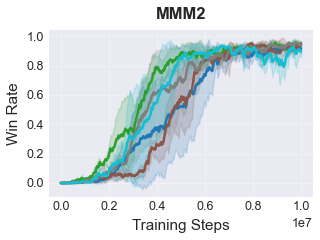}}
\subfloat[6h vs 8z]{\includegraphics[width=0.33\textwidth]{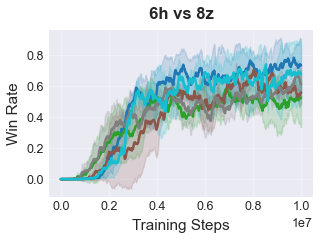}}
\subfloat[3s5z vs 3s6z]{\includegraphics[width=0.33\textwidth]{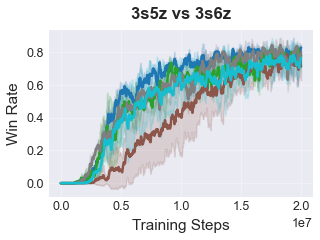}} \\
\subfloat[8$\times$1-Agent Ant]{\includegraphics[width=0.33\textwidth]{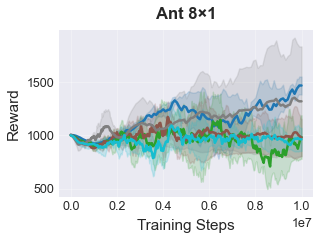}}
\subfloat[6$\times$1-Agent HalfCheetah]{\includegraphics[width=0.33\textwidth]{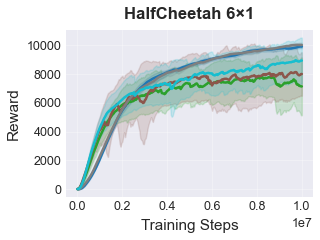}}
\subfloat[6$\times$1-Agent Walker2d]{\includegraphics[width=0.33\textwidth]{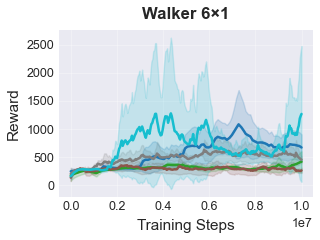}} \\
\subfloat[academy counterattack easy]{\includegraphics[width=0.33\textwidth]{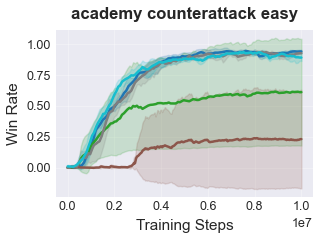}}
\subfloat[academy pass and shoot with keeper]{\includegraphics[width=0.33\textwidth]{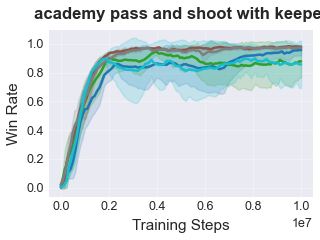}}
\subfloat[academy 3 vs 1 with keeper]{\includegraphics[width=0.33\textwidth]{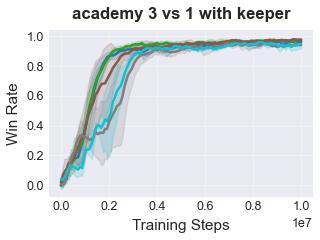}} \\
\caption{Ablation Study under 5 Random Seeds}
\label{fig:ablation}
\end{figure*}

The experimental results are shown in Fig. \ref{fig:ablation}. We observe that when directly using the last consensus instead of mixing all consensus outputs via our Actor-Compressor, the model performance decreases in most scenarios, suggesting that the decoder auto-regression process may lose some useful information from previous generations. This confirms the value of the Actor-Compressor as a mechanism for aggregating the full trajectory of consensus vectors rather than relying on the final one. Regarding the consensus iteration time, we find that selecting $n$ (the number of agents) as the iteration count is a proper choice: too few iterations may be insufficient to generate a good consensus among agents, while too many iterations may introduce excessive noise and redundant information, thereby increasing the training difficulty of the Actor-Compressor. Intuitively, choosing the number of iterations equal to $n$ aligns with the decoding process of MAT, allowing each agent to sufficiently adjust its action in response to the actions of others.

\section{Why the CMAT Architecture Is Effective}
\label{app:theory}

This appendix gives a rigorous, self-contained account of why the \emph{deterministic}-consensus architecture that we actually deploy is effective. In contrast to the more general analysis of Appendix~\ref{app:theory-future}, which is stated for a stochastic consensus policy and serves mainly to motivate future extensions, the results here make no probabilistic-coordination and no tabular-convergence assumptions. They are statements about the representational structure induced by conditioning every action head on a single deterministic consensus vector $c=\mu_\theta(\mathcal{O})$, and about the iterative refinement that produces $c$. Throughout, we separate what is genuinely a property of the architecture from what merely reflects added capacity, and we flag every assumption a reader might contest. The central thesis of this appendix is one of \emph{representational efficiency}: that iterative refinement with a single reused block can express the same coordination computation that an encoder-only controller can express only by stacking, and therefore by growing, its parameters. Every result here is accordingly a statement about what the network structure is able to represent at a given parameter budget, not about whether training converges to that representation; we do not claim any global-convergence or optimization guarantee for the deep model, and questions of trainability are settled by our experiments rather than by any proof below.

\paragraph{Scope of the theory: representation, not optimization.} To prevent a natural misreading, we state the boundary of these results explicitly. Every statement in this appendix is about \emph{representational capacity}: what the deterministic-consensus decoding structure can express, how well its best-case expressible policy approximates the optimum (Theorems~\ref{thm:collapse}--\ref{thm:subopt}), why iterative refinement is the right way to represent a self-consistent consensus (Theorem~\ref{thm:unroll}, Proposition~\ref{prop:meanpool}), and why weight-shared iteration represents the same coordination computation as depth at a smaller parameter budget (Theorem~\ref{thm:depth}). None of these results speaks to \emph{optimization}: whether a particular training procedure on our finite network actually finds the good representation the capacity results permit is a separate question, settled by our empirical study rather than by any proof here. The results are best read as structural justifications for the architecture, made precise, rather than as guarantees of algorithmic effectiveness.

We address the two questions that are specific to CMAT. The mitigation of the action-space Curse of Dimensionality is shared by any attention-based centralized controller and is not claimed as a contribution unique to CMAT; we nonetheless include it below because it is the scaffold on which the two CMAT-specific arguments rest. The two questions are: (i) why the consensus must be produced by \emph{iterative} refinement rather than a single aggregation such as mean-pooling; and (ii) why conditioning on a deterministic consensus is more effective than an encoder-only controller with independent action heads at a comparable parameter budget.

\subsection{Setup: Consensus-Separable Team Value Functions}
Recall the single-agent multi-action MDP of the Preliminaries, with team action $\mathcal{A}=(a^1,\dots,a^n)\in\bar{\mathcal{A}}=\prod_i\mathcal{A}^i$ and optimal team action-value $\mathrm{Q}^\star(\mathcal{O},\mathcal{A})$. The action-space Curse of Dimensionality is the statement that the greedy step $\max_{\mathcal{A}\in\bar{\mathcal{A}}}\mathrm{Q}^\star(\mathcal{O},\mathcal{A})$ ranges over $\prod_i|\mathcal{A}^i|$ joint actions, exponential in $n$. The CMAT decoder produces each action from its own observation embedding $\hat o^i$ and the shared consensus $c$ through Eq.~\eqref{eq:actor_MLP}. The representational premise that makes this structure exact is that the team value is additively separable once conditioned on a low-dimensional coordination statistic.

\begin{definition}[$d$-consensus-separability]
\label{def:sep}
A team value function $\mathrm{Q}(\mathcal{O},\mathcal{A})$ is \emph{$d$-consensus-separable} if there exist a coordination map $f:\mathcal{O}\to\mathcal{C}\subseteq\mathbb{R}^d$ and local value functions $q^i$ such that
\begin{equation}
\mathrm{Q}(\mathcal{O},\mathcal{A})=\sum_{i=1}^n q^i(\hat o^i,a^i,c),\qquad c=f(\mathcal{O}),\quad \forall\,\mathcal{O},\mathcal{A}.
\end{equation}
For a general $\mathrm{Q}$, the best $d$-dimensional consensus-separable approximation error is
\begin{equation}
\varepsilon_d(\mathrm{Q})=\inf_{f,\{q^i\}}\ \sup_{\mathcal{O},\mathcal{A}}\ \Bigl|\,\mathrm{Q}(\mathcal{O},\mathcal{A})-\sum_{i=1}^n q^i(\hat o^i,a^i,f(\mathcal{O}))\,\Bigr|.
\label{eq:epsd}
\end{equation}
\end{definition}

This is the learned, nonlinear, state-dependent analogue of the coupling signals used in classical operations-research decompositions of large-scale sequential decision problems: the Lagrangian (dual price) in weakly coupled MDPs and restless bandits \citep{whittle1988restless, adelman2008relaxations}, and the population mean field in mean-field control \citep{carmona2018probabilistic}. In those methods a shared signal, a price $\lambda$ or a first moment, renders the $n$ per-component problems conditionally independent. Here $c=f(\mathcal{O})$ plays the same decoupling role but is computed by attention on the true state and is not restricted to be linear or to a first moment. The quantity $\varepsilon_d$ in Eq.~\eqref{eq:epsd} is a deterministic, distribution-free measure of how much coordination the problem intrinsically requires, and it is non-increasing in $d$.

\subsection{Collapse of the Exponential Joint Maximization}
\begin{theorem}[Argmax collapse]
\label{thm:collapse}
If $\mathrm{Q}$ is $d$-consensus-separable (Def.~\ref{def:sep}), then for every $\mathcal{O}$,
\begin{equation}
\max_{\mathcal{A}\in\bar{\mathcal{A}}}\mathrm{Q}(\mathcal{O},\mathcal{A})=\sum_{i=1}^n\max_{a^i\in\mathcal{A}^i}q^i(\hat o^i,a^i,f(\mathcal{O})),
\end{equation}
and $\mathcal{A}^\star(\mathcal{O})=\bigl(\arg\max_{a^i}q^i(\hat o^i,a^i,f(\mathcal{O}))\bigr)_{i=1}^n$ is a joint maximizer. Hence the greedy team action is obtained from one evaluation of the consensus $c=f(\mathcal{O})$ followed by $n$ independent local maximizations, at cost $O\!\left(\sum_i|\mathcal{A}^i|\right)$ rather than $O\!\left(\prod_i|\mathcal{A}^i|\right)$.
\end{theorem}
\begin{proof}
Fix $\mathcal{O}$ and set $c=f(\mathcal{O})$. The objective $\sum_i q^i(\hat o^i,a^i,c)$ is a sum whose $i$-th term depends on the decision variable $a^i$ only, since the embeddings $\hat o^i$ and the shared $c$ are determined once $\mathcal{O}$ is fixed. Maximization of a separable sum over a product domain decomposes coordinatewise, $\max_{\mathcal{A}}\sum_i q^i=\sum_i\max_{a^i}q^i$, attained at the tuple of coordinatewise maximizers. Each inner maximization enumerates $|\mathcal{A}^i|$ actions, giving the stated cost.
\end{proof}

\begin{remark}[Exploiting structure, not defeating the curse]
Theorem~\ref{thm:collapse} does not eliminate the Curse of Dimensionality unconditionally. When $\varepsilon_d(\mathrm{Q}^\star)>0$ the collapse is only approximate (Theorem~\ref{thm:subopt}). The precise claim is that CMAT \emph{exploits} two structural properties, permutation symmetry across agents and a low-dimensional coupling channel, to trade an exponential joint maximization for a linear one, exactly the trade made by the classical decomposition methods above. CMAT's contribution is to \emph{learn} the coupling statistic $c$ rather than hand-design it.
\end{remark}

\subsection{Coordination Dimension versus Suboptimality}
\begin{theorem}[Suboptimality of consensus decoding, conditional on realizability]
\label{thm:subopt}
\emph{(Realizability premise.)} Assume the optimal team value admits a $d$-consensus-separable approximation with error $\varepsilon_d(\mathrm{Q}^\star)$, and that the CMAT decoder actually realizes such a best approximation, that is, there exist a coordination map $f$ and heads $q^i$ within the decoding structure attaining the infimum in Eq.~\eqref{eq:epsd}. \emph{(Conclusion.)} Let $\mathrm{Q}^\star$ be the optimal team action-value and let $\widehat{\mathrm{Q}}(\mathcal{O},\mathcal{A})=\sum_i q^i(\hat o^i,a^i,f(\mathcal{O}))$ be this approximation, so that $\|\widehat{\mathrm{Q}}-\mathrm{Q}^\star\|_\infty=\varepsilon_d(\mathrm{Q}^\star)$. Let $\hat\pi$ be the deterministic policy that plays the greedy team action of $\widehat{\mathrm{Q}}$, that is, the collapse of Theorem~\ref{thm:collapse}. Then for every state,
\begin{equation}
V^{\hat\pi}(\mathcal{O})\ \ge\ V^\star(\mathcal{O})-\frac{2\,\varepsilon_d(\mathrm{Q}^\star)}{1-\gamma}.
\end{equation}
\end{theorem}
\begin{proof}
By Theorem~\ref{thm:collapse}, $\hat\pi$ selects $\arg\max_{\mathcal{A}}\widehat{\mathrm{Q}}(\mathcal{O},\mathcal{A})$, so it is greedy with respect to $\widehat{\mathrm{Q}}$. The classical bound on greedy policies of approximate action-value functions \citep{singh1994upper, bertsekas1996neuro} states that $\|\widehat{\mathrm{Q}}-\mathrm{Q}^\star\|_\infty\le\varepsilon$ implies the greedy policy of $\widehat{\mathrm{Q}}$ is $\tfrac{2\varepsilon}{1-\gamma}$-suboptimal. Taking $\varepsilon=\varepsilon_d(\mathrm{Q}^\star)$ gives the claim.
\end{proof}

The greedy policy $\hat\pi$ is representable within the CMAT decoding structure, since each action head is a function of $(\hat o^i,c)$ alone. Because $\varepsilon_d(\mathrm{Q}^\star)$ is non-increasing in $d$ and vanishes once $d$ reaches the intrinsic coordination dimension of the problem, Theorem~\ref{thm:subopt} both certifies the deterministic architecture (small $\varepsilon_d$ implies near-optimal decoding) and gives quantitative guidance for choosing the consensus width $d$. It also pinpoints where CMAT should excel. In large-scale operational systems such as matching, queueing, fleet and capacity allocation, and congestion pricing, inter-agent coupling is transmitted mainly through low-dimensional aggregates such as aggregate supply-demand, prices, or congestion levels, so $\varepsilon_d$ decays quickly in $d$ and a modest consensus suffices even for large $n$. We stress that this is a conditional, ``good representation implies good policy'' result: it presumes the realizability premise, that the learned $f$ and $q^i$ reach the best $d$-consensus-separable fit with error $\varepsilon_d(\mathrm{Q}^\star)$. It is a statement about the representational capacity of the decoding structure, not a claim that PPO drives the realized approximation error down to $\varepsilon_d(\mathrm{Q}^\star)$; whether training actually attains this fit is an empirical question, addressed by our experiments rather than by this bound.

\subsection{Why Iterate: Consensus as Unrolled Coordination Optimization}
The consensus that supports Theorem~\ref{thm:collapse} cannot be an arbitrary one-shot summary of $\mathcal{O}$. The local maximizers $\hat a^i(c)=\arg\max_{a^i}q^i(\hat o^i,a^i,c)$ depend on $c$, so for the decomposition to be self-consistent the coordination statistic must agree with the actions it induces. A valid consensus is therefore a fixed point
\begin{equation}
c^\star=\Phi\bigl(\mathcal{O},(\hat a^i(c^\star))_{i=1}^n\bigr)=:\mathcal{T}_{\mathcal{O}}(c^\star),
\label{eq:fixedpoint}
\end{equation}
where $\Phi$ re-aggregates the induced per-agent responses into the coordination statistic. This is the self-consistency that defines equilibrium aggregates in mean-field and weakly coupled formulations. A single feed-forward aggregation, whether the initial vector $e^0$ or any mean-pooling, computes $\mathcal{T}_{\mathcal{O}}(c_0)$ once from an uninformed $c_0$ and does not solve Eq.~\eqref{eq:fixedpoint}.

We therefore model the decoder as an unrolled optimizer. We adopt this view rather than a contraction argument because it does not rely on the contraction constant of attention, which is not automatically below one (Remark~\ref{rem:contraction}). We state the view as an explicit modeling assumption rather than a derived property, and we are careful about the direction of the claim: we posit that the decoder can be read as unrolling a coordination solver, not that a trained attention block provably performs exact gradient ascent on a specific potential.

\begin{assumption}[Unrolled-optimization view of the decoder]
\label{ass:unroll}
We model the refinement as follows. We posit a coordination potential $J(\cdot\,;\mathcal{O}):\mathbb{R}^d\to\mathbb{R}$, bounded above and $L$-smooth in $c$, whose stationary points are the self-consistent consensus vectors of Eq.~\eqref{eq:fixedpoint}, and we treat one decoder refinement step as implementing a gradient (or proximal) ascent step $c_{k+1}=c_k+\eta_k\nabla_c J(c_k;\mathcal{O})$ with $\eta_k\le 1/L$. This is an interpretive assumption about the computation the block is designed and trained to carry out, not a proven identity between the softmax-attention map and $\nabla_c J$.
\end{assumption}

This is the standard modeling premise of the algorithm-unrolling, learning-to-optimize, and deep-equilibrium literatures \citep{gregor2010learning, marino2018iterative, bai2019deep}, in which a recurrent, weight-shared block is designed by unrolling the iterations of a solver for an implicitly defined objective and is then trained end to end. Following that tradition, we invoke the assumption only in the forward direction, to motivate why iterating helps and to predict the qualitative behaviour that follows (monotone-then-saturating improvement, Theorem~\ref{thm:unroll}). We do not claim the converse, that the learned decoder recovers the exact ascent dynamics of any particular $J$; the empirical saturation reported in Appendix~\ref{sec:ablation} is the evidence we offer that the modeled behaviour actually obtains.

\begin{theorem}[Monotone improvement of iteration; one-shot is insufficient at a fixed width]
\label{thm:unroll}
Under Assumption~\ref{ass:unroll}, the decoder iterates satisfy
\begin{equation}
J(c_{k+1};\mathcal{O})\ \ge\ J(c_k;\mathcal{O})+\tfrac{\eta_k}{2}\|\nabla_c J(c_k;\mathcal{O})\|^2,
\end{equation}
so $J(c_k;\mathcal{O})$ is non-decreasing in $k$ and converges, and for constant step size $\min_{k\le m}\|\nabla_c J(c_k;\mathcal{O})\|^2= O(1/m)$. In particular the one-shot consensus $e^0$ corresponds to $k=1$ and is generally not stationary, incurring a strictly positive gap $J^\star-J(e^0;\mathcal{O})>0$ whenever $\nabla_c J(c_0;\mathcal{O})\neq 0$; this gap is reduced monotonically by further iterations.
\end{theorem}
\begin{proof}
For an $L$-smooth $J$ and a step $c_{k+1}=c_k+\eta_k\nabla_c J(c_k)$ with $\eta_k\le 1/L$, the descent lemma in ascent form gives $J(c_{k+1})\ge J(c_k)+\eta_k(1-\tfrac{L\eta_k}{2})\|\nabla_c J(c_k)\|^2\ge J(c_k)+\tfrac{\eta_k}{2}\|\nabla_c J(c_k)\|^2$ \citep{nesterov2018lectures}. Hence $J(c_k;\mathcal{O})$ is non-decreasing and, being bounded above, converges. Summing the inequality over $k\le m$ and dividing by $\sum_{k}\eta_k$ bounds $\min_{k\le m}\|\nabla_c J(c_k)\|^2$ by $2(J^\star-J(c_0))/(\sum_{k}\eta_k)=O(1/m)$ for constant $\eta_k$. Since $e^0=c_1$ is a single ascent step from an uninformed $c_0$, it is stationary only if $\nabla_c J(c_0)=0$, which fails generically.
\end{proof}

This matches the behaviour observed empirically in Appendix~\ref{sec:ablation}: replacing the iterated consensus by $e^0$ (no iteration) degrades performance, and increasing the number of iterations $m$ yields diminishing returns that saturate, consistent with the non-decreasing and converging sequence $J(c_k;\mathcal{O})$.

\begin{remark}[Contraction alternative and an honest caveat]
\label{rem:contraction}
If $\mathcal{T}_{\mathcal{O}}$ is a contraction of modulus $\rho<1$, then it has a unique fixed point $c^\star$ and $\|c_k-c^\star\|\le\rho^k\|c_0-c^\star\|$, a geometric convergence that likewise predicts the empirical saturation. We do not rely on this route in the main statement because a standard softmax attention block is not guaranteed to be a contraction; its Lipschitz constant can exceed one \citep{kim2021lipschitz}. Enforcing contraction requires spectral or normalization constraints, or restricting attention to a neighborhood of $c^\star$ (local contraction). The unrolled-optimization view of Theorem~\ref{thm:unroll} avoids this dependence.
\end{remark}

\begin{proposition}[Mean-pooling captures only a first-order moment]
\label{prop:meanpool}
A single mean-pooling aggregation produces $c=\psi\bigl(\tfrac1n\sum_i\phi(\hat o^i)\bigr)$, a function of the empirical first moment of the embeddings, which is precisely the mean-field statistic and is exact only in the exchangeable, weak-coupling, large-$n$ limit \citep{carmona2018probabilistic}. If the self-consistent consensus $c^\star$ of Eq.~\eqref{eq:fixedpoint} depends on higher-order or pairwise structure of $\{\hat o^i\}$, then at any fixed embedding width mean-pooling cannot represent $c^\star$, whereas the data-dependent weighted aggregation of attention, iterated, can. Quantitatively, exact representation of permutation-invariant set functions by sum or mean pooling requires latent width at least the set cardinality \citep{wagstaff2019limitations}, a blow-up that iterated cross-attention avoids.
\end{proposition}
\begin{proof}
Write the aggregation as $c=\psi(\mu)$ with $\mu=\tfrac1n\sum_i\phi(\hat o^i)$. The map factors through $\mu$, so it is constant on every level set $\{\{\hat o^i\}:\tfrac1n\sum_i\phi(\hat o^i)=\mu\}$. Take two configurations $\{\hat o^i\}$ and $\{\tilde o^i\}$ with equal first moment, $\tfrac1n\sum_i\phi(\hat o^i)=\tfrac1n\sum_i\phi(\tilde o^i)$, but differing in some higher-order statistic; both are mapped to the same $c=\psi(\mu)$. If the self-consistent consensus $c^\star$ of Eq.~\eqref{eq:fixedpoint} takes distinct values on these two configurations, then no function of $\mu$ can equal $c^\star$ on both, so mean-pooling cannot represent $c^\star$; this holds at every fixed width because the factorization through $\mu$ is width-independent. For the quantitative claim, \citet{wagstaff2019limitations} show that a sum- or mean-decomposition $\psi(\sum_i\phi(\hat o^i))$ represents all continuous permutation-invariant functions of an $n$-element set exactly only if the latent dimension of $\phi$ is at least $n$; below this width some set functions are unrepresentable. Iterated cross-attention is not of this fixed-moment form: each step reweights the tokens by data-dependent, non-uniform attention scores and composes these reweightings across steps, so it realizes functions that depend on pairwise and higher-order structure of $\{\hat o^i\}$ without a latent width that grows with $n$.
\end{proof}

\subsection{Why Consensus Beats an Encoder-Only Controller at a Fixed Budget}
We first dispose of a claim we do \emph{not} make. Under full observability the encoder's self-attention already makes every embedding $\hat o^i$ a function of the entire $\mathcal{O}$, so with unbounded depth and width an encoder-only controller with independent heads is also a universal approximator of the optimal deterministic map $\mathcal{A}^\star(\mathcal{O})$. There is therefore no unconditional, infinite-capacity representational separation between the two, and we do not assert one. The advantage of consensus is a fixed-budget phenomenon: it attains a given coordination accuracy with far fewer parameters, because it buys the required number of coordination rounds through weight-shared iteration rather than through depth.

Before stating the result, we reconcile it with the previous subsection, which might read as being in tension with it: there we argued that iteration is necessary, while here we concede that an encoder-only controller can also be universal. The two claims live on the same axis and do not conflict. Both are fixed-budget statements. When we said a single aggregation is insufficient (Proposition~\ref{prop:meanpool}, Theorem~\ref{thm:unroll}), the quantifier was \emph{at a fixed width}: with the latent width allowed to grow, a one-shot map can also represent the fixed point, exactly as the infinite-capacity universality conceded here. Conversely, a deep encoder-only controller does not avoid iteration; a stack of $k$ layers performs $k$ rounds of refinement, only with $k$ independent sets of weights rather than one reused block. Multi-round coordination therefore happens in both architectures; what differs is whether the rounds are paid for in parameters (depth) or in inference compute (weight-shared iteration). The universality of an infinite encoder is thus not a case of ``no iteration needed'' but of iteration realized inefficiently, by spending unbounded depth.

This distinction has practical, not merely asymptotic, weight in our setting. A recurring empirical finding in deep RL, and in MARL specifically, is that naively enlarging the network degrades rather than improves learning: value bootstrapping, non-stationary targets, and a moving data distribution make large models unstable and hard to train, and additional capacity often has to be actively stabilized before it helps at all \citep{ota2021training, bjorck2021towards, nikishin2022primacy}. Under this regime, an architecture that reaches a given coordination accuracy with a \emph{smaller} parameter count, by reusing one block across refinement steps rather than stacking depth, is not only parameter-efficient in the abstract but easier to optimize in practice, which is the regime our experiments operate in.

The result below rests on three premises, which we state explicitly. First, coordination is a fixed-point computation solved by repeated application of a refinement operator $\mathcal{T}_{\mathcal{O}}$, and it converges geometrically at some rate $\rho<1$, so accuracy $\epsilon$ is reached after $K(\epsilon)=\Theta(\log(1/\epsilon)/\log(1/\rho))$ applications. Second, $\mathcal{T}_{\mathcal{O}}$ is nonlinear and state-dependent, so its $k$-fold composition cannot be collapsed into a single map computed at lower cost. Third, the point of comparison is an encoder that uses independent per-layer weights, which is exactly the standard Transformer encoder. Because this is a statement about representational capacity, we take the shared decoder block to have the capacity to represent one application of $\mathcal{T}_{\mathcal{O}}$, so that $K$ iterations of the block represent the $K$-fold composition $\mathcal{T}_{\mathcal{O}}^{(K)}$. This is a realizability premise of the same kind as elsewhere in this appendix, and it is well founded structurally: the block's input-output form is built to match one step of Eq.~\eqref{eq:fixedpoint} (attend to the current $c$ to form the induced responses, then re-aggregate them into the next $c$), and since $\mathcal{T}_{\mathcal{O}}$ is a \emph{stationary} operator, applied identically at every step, a single reused block is its natural representation. We emphasize that this result does \emph{not} invoke Assumption~\ref{ass:unroll}: it is stated directly in terms of the fixed-point operator $\mathcal{T}_{\mathcal{O}}$ of Eq.~\eqref{eq:fixedpoint} and does not require the decoder to implement ascent on any potential $J$. The rate $\rho$ is a modeling assumption on the task and operator, not a proven property of the softmax attention block, which need not be a contraction (Remark~\ref{rem:contraction}).

\begin{theorem}[Parameter efficiency of weight-shared coordination]
\label{thm:depth}
Under the three premises above, an encoder-only network with independent per-layer weights reaches coordination accuracy $\epsilon$ only at a cost of $\Theta(K(\epsilon))$ layers, and hence $\Theta(K(\epsilon))$ times the per-layer parameters, whereas CMAT's decoder reuses one block across refinement steps and reaches the same accuracy with a \emph{fixed} parameter count and $K(\epsilon)$ times the inference compute. As the coordination problem stiffens ($\rho\to1$, $K(\epsilon)\to\infty$), the encoder-only parameter count diverges while CMAT's stays constant.
\end{theorem}
\begin{proof}
We separate the two quantities the statement involves: the number of refinement rounds $K(\epsilon)$, and the parameter cost of realizing those rounds.

\emph{Rounds.} That $K(\epsilon)=\Theta(\log(1/\epsilon)/\log(1/\rho))$ is the definition of geometric convergence and is exhibited by any contractive instance. For concreteness take the linear operator $\mathcal{T}_{\mathcal{O}}(c)=Wc+b$ with $\|W\|=\rho<1$ and fixed point $c^\star=(I-W)^{-1}b$; then $c_k-c^\star=W^k(c_0-c^\star)$, so $\|c_k-c^\star\|=\Theta(\rho^k)$ for generic $c_0$, and $\rho^k\le\epsilon$ requires $k\ge\log(1/\epsilon)/\log(1/\rho)=K(\epsilon)$. This instance fixes the \emph{rate}; it is not used to argue the parameter cost, since a linear $W^k$ is itself a single matrix and would collapse to one layer.

\emph{Parameters.} The parameter cost turns on non-compressibility, which is where the second premise enters. When $\mathcal{T}_{\mathcal{O}}$ is nonlinear and depends on $\mathcal{O}$ and on the current induced actions, its $k$-fold composition $\mathcal{T}_{\mathcal{O}}^{(k)}$ is a genuinely different nonlinear map for each $k$ and cannot be pre-multiplied into a single lower-cost operator the way $W^k$ can. An encoder-only network with independent per-layer weights can represent $k$ rounds of such a composition only by instantiating $k$ distinct nonlinear layers, one per round, so matching accuracy $\epsilon$ costs $\Theta(K(\epsilon))$ layers and therefore $\Theta(K(\epsilon))$ times the per-layer parameter count. A decoder that reuses a single block representing $\mathcal{T}_{\mathcal{O}}$ across steps (third premise) represents the same $k$-fold composition with one block of parameters applied $k$ times, i.e.\ a fixed parameter count and $K(\epsilon)$ times the sequential inference compute. The reuse is legitimate precisely because $\mathcal{T}_{\mathcal{O}}$ is the same stationary map at every step, so representing $\mathcal{T}_{\mathcal{O}}^{(K)}$ requires representing only one operator, not $K$ of them. Taking $\rho\to1$ sends $K(\epsilon)\to\infty$, so the independent-weight parameter count diverges while the weight-shared count is unchanged.
\end{proof}

Were one instead to tie the encoder's layer weights, it would itself become a weight-shared iteration, which is the very ingredient we identify as useful; the comparison is therefore between depth-parameterized and iteration-parameterized coordination, not between CMAT and attention as such.

\begin{remark}[No redundant replication]
If every agent must condition on the same global summary $s(\mathcal{O})$, an encoder-only controller must carry $s(\mathcal{O})$ inside each of the $n$ token streams, a width cost that scales with $n$, whereas consensus computes $s(\mathcal{O})$ once into $c$ and broadcasts it. This is a representation-size argument in CMAT's favor, independent of Theorem~\ref{thm:depth}.
\end{remark}

\begin{remark}[Consistency as an inductive bias]
The fixed point of Eq.~\eqref{eq:fixedpoint} requires all agents to decode from one self-consistent global state. Conditioning every head on a single $c$ builds this consistency in by construction, whereas an encoder-only controller lets the $n$ token streams evolve separately and stop interacting before the action heads (the point-wise head applied to $\hat o^i$ performs no further cross-agent mixing), so joint consistency must instead be learned. We state this as an optimization and inductive-bias advantage, not an impossibility: with enough capacity and data an encoder-only model can learn consistency, but the consensus structure does not have to.
\end{remark}

\begin{remark}[Design guidance for Transformer-based MARL]
\label{rem:design-guidance}
Theorem~\ref{thm:depth} is stated for CMAT but its content is architectural and transfers to Transformer-based cooperative MARL more broadly. It suggests three principles. First, when coordination is genuinely iterative, in the sense of requiring several rounds of mutual adjustment before a joint decision stabilizes (iterated best response, message passing, consensus formation), buying those rounds through a weight-shared recurrent block is more parameter-efficient than buying them through stacked depth, so a weight-tied refinement module is preferable to simply deepening an encoder. Second, because the two costs are decoupled, the number of refinement steps becomes a free inference-time knob: a model can adapt its computation to the difficulty of the current state, spending more iterations on stiff coordination problems and fewer on easy ones, without changing its parameter count. Third, the stiffness $\rho$, or equivalently the coordination rank of Appendix~\ref{sec:coord-rank}, is a diagnostic for when this helps at all: tasks with small $K(\epsilon)$ gain little from iteration, and the benefit of weight-shared refinement is largest precisely when coordination is rigid ($\rho\to1$). These principles are conditional in the same way the theorem is: they presume coordination is an iterative fixed-point computation and that the refinement block represents a (at least locally) contractive operator; where a single aggregation already suffices, as in the weak-coupling, large-$n$ regime of Proposition~\ref{prop:meanpool}, iteration is not expected to pay off. As with the rest of this appendix, the guidance concerns representation capacity and parameter efficiency.
\end{remark}

\subsection{Synthesis}
Points (i) and (ii) are the same mechanism seen from two sides. The consensus is a shared, weight-shared, iteratively refined solver for the coordination fixed point of Eq.~\eqref{eq:fixedpoint}. Iterating it is how the fixed point is reached (Theorem~\ref{thm:unroll}, Proposition~\ref{prop:meanpool}), and reaching it by weight-shared iteration rather than by depth is why it dominates an encoder-only controller at a fixed budget (Theorem~\ref{thm:depth}). The additive decomposition these arguments support is what collapses the exponential joint decision into $n$ local ones with a controlled suboptimality (Theorems~\ref{thm:collapse}--\ref{thm:subopt}). None of these statements requires stochasticity in the consensus, tabular assumptions, or a convergence guarantee for the deep model.

\section{A More General Framework: Stochastic Consensus and Future Directions}
\label{app:theory-future}

The analysis in this appendix concerns a more general and more powerful version of CMAT than the one we deploy. It endows the consensus with its own stochastic policy $\pi^c(c\mid\mathcal{O})$ and casts training as a bilevel (cooperative Stackelberg) problem, which enlarges the representable policy class beyond what a deterministic consensus can reach (Appendix~\ref{sec:coord-rank}). We present it as a future direction. The rigorous justification of the deterministic architecture we actually use is given separately in Appendix~\ref{app:theory}; the present appendix instead explains, at a higher level, why the richer stochastic-consensus generalization is effective, why our deterministic method is a tractable simplification of it, and what additional mechanisms it would require to remain stable in practice.

This appendix provides a theoretical interpretation of CMAT's hierarchical design under simplified (tabular) conditions, building on established convergence results in tabular and linear settings. The analysis clarifies how the latent consensus mechanism eliminates the order-dependent bias inherent in sequential MARL formulations such as MAT. We emphasize at the outset that it does {not} constitute a formal convergence proof for our deep, deterministic-consensus implementation; rather, it illustrates the structural principles that motivate CMAT and outlines a more general formulation that we leave to future work. In particular, the most general analysis below is stated for a \emph{stochastic} consensus policy $\pi^c(c\mid\mathcal{O})$ and a cooperative Stackelberg (bilevel) view. Our implemented model uses a deterministic consensus $c=\mu_\theta(\mathcal{O})$, which is a special case of this formulation (Remark~\ref{rem:stochastic-deterministic}). We therefore treat the stochastic-consensus and full bilevel Stackelberg optimization as a promising extension rather than as a description of the deployed method, and we note that realizing it in practice raises additional training-stability challenges (e.g., variance from sampling in the latent space and the need for variance-reduction techniques) that we leave for future work.

\subsection{Preliminary Assumptions}
For rigorous statements we adopt standard assumptions: finite observation space $\mathcal{O}$ and action spaces $\mathcal{A}^i$; a finite consensus set $\mathcal{C}$; softmax policy parameterization with full support; Robbins-Monro step sizes; and ergodicity of the induced Markov chain \citep{agarwal2021theory, mei2020global, chen2024convergence}.

\subsection{Reformulation as a Cooperative Stackelberg Game}
CMAT models the joint policy as
\begin{equation}
\pi(\mathcal{A}\mid\mathcal{O}) = \int \pi^c(c\mid\mathcal{O}) \prod_{i=1}^{n} \pi^{i}\bigl(a^{i}\mid\mathcal{O}, c\bigr) \, dc,
\label{eq:joint-stochastic}
\end{equation}
where $\pi^c(c\mid\mathcal{O})$ is the consensus policy (stochastic in general), and each $\pi^{i}$ is the action policy for agent $i$ conditioned on $\mathcal{O}$ and $c$. The objective is the discounted cumulative reward $J(\pi^c,\pi^{1:n})$.

\begin{remark}[Stochastic vs. Deterministic Consensus]
\label{rem:stochastic-deterministic}
The integral formulation in Eq.~\eqref{eq:joint-stochastic} provides the most general theoretical framework. In our implementation, we employ a deterministic consensus generator $c = \mu_\theta(\mathcal{O})$, which corresponds to $\pi^c(c\mid\mathcal{O}) = \delta(c - \mu_\theta(\mathcal{O}))$. Under this specialization, the joint policy reduces to
\begin{equation}
\pi(\mathcal{A}\mid\mathcal{O}) = \prod_{i=1}^n \pi^i(a^i\mid\mathcal{O}, \mu_\theta(\mathcal{O})),
\label{eq:joint-deterministic}
\end{equation}
which matches Eq. \ref{eq:hier-mdp} in the main paper. While this restricts the mixture to a single component for any given $\mathcal{O}$, the expressiveness is recovered by allowing $\mu_\theta$ to vary across different observations and by the high dimensionality of the consensus vector $c \in \mathbb{R}^d$. In practice, this deterministic choice simplifies the ratio computation (Eq.\ref{eq:cmat} in the main paper) and avoids the need for importance weighting over the latent space.
\end{remark}

This hierarchical structure can be viewed as a cooperative Stackelberg game: the leader (consensus policy $\pi^c$) commits to a latent strategy $c$, anticipating that the followers (action policies $\pi^{1:n}$) will best respond. Because all agents share the same reward, the game is fully cooperative. Stackelberg equilibria are known to be more Pareto-efficient than Nash equilibria in such settings:

\begin{proposition}[Pareto Improvement over Nash Equilibrium]
\label{prop:stackelberg-pareto}
In a cooperative Markov game with identical rewards, any Stackelberg equilibrium $(\pi^{c*}, \pi^{1:n*})$ achieves a payoff at least as high as any Nash equilibrium payoff $J_{\text{NE}}$ of the simultaneous-move game:
\begin{equation}
J(\pi^{c*}, \pi^{1:n*}) \geq J_{\text{NE}}.
\end{equation}
The inequality is strict when the Nash equilibrium is Pareto-suboptimal and the leader's commitment power is sufficient to steer the followers away from it.
\end{proposition}

\begin{proof}
In a cooperative (identical-interest) game, all agents share the objective $J$. A Stackelberg equilibrium $(\pi^{c*}, \pi^{1:n*})$ satisfies $\pi^{c*} \in \arg\max_{\pi^c} J(\pi^c, \text{BR}(\pi^c))$, where $\text{BR}(\pi^c)$ denotes the followers' best response.

Any Nash equilibrium $(\bar{\pi}^{1:n})$ of the simultaneous-move game can be recovered in the Stackelberg framework as follows: let $\pi^c$ assign all probability to an arbitrary fixed consensus $\bar{c}$. Since $\bar{c}$ is uninformative, the followers' decision problem reduces to the original simultaneous-move game. By the definition of Nash equilibrium, $(\bar{\pi}^{1:n})$ is a mutual best response, so the payoff under this trivial consensus is exactly $J_{\text{NE}}$. 

Since the Stackelberg leader optimizes over {all} possible $\pi^c$, including this trivial choice, the Stackelberg payoff must be at least the Nash payoff: $J(\pi^{c*}, \pi^{1:n*}) \geq J_{\text{NE}}$. Strict inequality occurs when the Nash equilibrium is Pareto-suboptimal and the leader can commit to an informative consensus that induces a better response \citep{kononen2008reinforcement, fiez2020implicit}.
\end{proof}

\subsection{The Policy-Class Ladder and Coordination Rank}
\label{sec:coord-rank}
The reason a \emph{stochastic} consensus is strictly more powerful than the deterministic one we deploy can be stated exactly. Fix a state $\mathcal{O}$ and view a joint action distribution as a nonnegative tensor $P\in\mathbb{R}_{\ge0}^{|\mathcal{A}^1|\times\cdots\times|\mathcal{A}^n|}$ with $\sum P=1$. Define three policy classes,
\begin{align}
\Pi_{\text{ind}}&=\Bigl\{\textstyle\prod_i\pi^i(a^i\mid\mathcal{O})\Bigr\},\notag\\
\Pi_{\text{con}}(d)&=\Bigl\{\textstyle\sum_{c=1}^{d}\pi^c(c\mid\mathcal{O})\prod_i\pi^i(a^i\mid\mathcal{O},c)\Bigr\},\\
\Pi_{\text{auto}}&=\Bigl\{\textstyle\prod_i\pi(a^i\mid a^{<i},\mathcal{O})\Bigr\},\notag
\end{align}
where $\Pi_{\text{con}}(d)$ uses a consensus taking $d$ discrete values. By the chain rule $\Pi_{\text{auto}}$ is the full simplex over joint actions (the class MAT represents).

\begin{proposition}[Coordination-rank ladder]
\label{prop:ladder}
For a fixed $\mathcal{O}$, a joint distribution $P$ belongs to $\Pi_{\text{con}}(d)$ if and only if its nonnegative tensor rank satisfies $\mathrm{rank}_+(P)\le d$. Consequently
\begin{equation}
\Pi_{\text{ind}}=\Pi_{\text{con}}(1)\subsetneq\Pi_{\text{con}}(2)\subsetneq\cdots\subsetneq\Pi_{\text{con}}(d)\subsetneq\cdots\subseteq\Pi_{\text{auto}},
\end{equation}
with equality $\Pi_{\text{con}}(d)=\Pi_{\text{auto}}$ once $d\ge\prod_i|\mathcal{A}^i|$.
\end{proposition}
\begin{proof}
View a joint distribution over $\mathcal{A}=(a^1,\dots,a^n)$ as an order-$n$ nonnegative tensor $P\in\mathbb{R}_{\ge0}^{|\mathcal{A}^1|\times\cdots\times|\mathcal{A}^n|}$ with entries summing to one. An element of $\Pi_{\text{con}}(d)$ has the form $P=\sum_{c=1}^{d}\pi^c(c\mid\mathcal{O})\prod_i\pi^i(\cdot\mid\mathcal{O},c)$. Each summand is a nonnegative scalar times an outer product $\bigotimes_i\pi^i(\cdot\mid\mathcal{O},c)$ of probability vectors, i.e.\ a nonnegative rank-one tensor, and the coefficients $\pi^c(\cdot\mid\mathcal{O})$ are nonnegative and sum to one, so $P$ is a convex combination of $d$ nonnegative rank-one tensors. By the definition of nonnegative tensor rank $\mathrm{rank}_+$ as the least number of nonnegative rank-one terms summing to $P$, membership $P\in\Pi_{\text{con}}(d)$ is equivalent to $\mathrm{rank}_+(P)\le d$ (any such decomposition can be renormalized so that the coefficients form a distribution and each factor is a probability vector, since all terms are nonnegative). This proves the characterization.

The chain follows. $\Pi_{\text{ind}}$ is the set of single product distributions, which is exactly $\mathrm{rank}_+(P)=1$, so $\Pi_{\text{ind}}=\Pi_{\text{con}}(1)$. The inclusions $\Pi_{\text{con}}(d)\subseteq\Pi_{\text{con}}(d+1)$ are immediate (append a zero-weight term). They are strict below saturation because nonnegative rank is a strict hierarchy: for each $d$ smaller than the maximal attainable nonnegative rank there exist nonnegative tensors of rank exactly $d+1$, which lie in $\Pi_{\text{con}}(d+1)\setminus\Pi_{\text{con}}(d)$. Finally, any $P$ is a mixture of at most $\prod_i|\mathcal{A}^i|$ degenerate product distributions, one placing all mass on each joint action, so $\mathrm{rank}_+(P)\le\prod_i|\mathcal{A}^i|$ and $\Pi_{\text{con}}(d)=\Pi_{\text{auto}}$ once $d\ge\prod_i|\mathcal{A}^i|$, since $\Pi_{\text{auto}}$ is the full simplex over joint actions.
\end{proof}

This motivates a problem-intrinsic scalar, the \emph{coordination rank} $\mathrm{crank}_\varepsilon(\pi^\star)=\min\{d:\exists\,\pi\in\Pi_{\text{con}}(d),\ \|\pi-\pi^\star\|\le\varepsilon\}$, which measures how much coordination a task genuinely requires and can be estimated empirically from the effective rank of the learned consensus. Marginalizing the shared latent $c$ correlates the $a^i$, so $c$ acts as a learned correlation device, the analogue of the shared signal in a correlated equilibrium \citep{aumann1974subjectivity}, which is entirely absent in $\Pi_{\text{ind}}$. We stress the boundary with Appendix~\ref{app:theory}: the deterministic model we deploy is, for a fixed $\mathcal{O}$, the rank-one end $\Pi_{\text{con}}(1)$ of this ladder (Remark~\ref{rem:expressiveness-scope}), and its effectiveness rests on the value-function decomposition proved there, not on this ladder. The ladder characterizes the strictly larger class unlocked only by a stochastic consensus, which is why we treat it as future work.

\subsection{Theoretical Justification 1: The Leader's Problem as a Finite MDP}
When the action policies $\pi^{1:n}$ are fixed, the leader faces a finite MDP with state space $\mathcal{O}$, action space $\mathcal{C}$, transition
\begin{equation}
P(\mathcal{O}'\mid\mathcal{O},c)=\mathbb{E}_{\pi^{1:n}}\bigl[P_{\text{env}}(\mathcal{O}'\mid\mathcal{O},\mathcal{A})\mid c\bigr],
\label{eq:leader-transition}
\end{equation}
and reward $r(\mathcal{O},c)=\mathbb{E}_{\pi^{1:n}}[R(\mathcal{O},\mathcal{A})\mid c]$. For this finite MDP, standard RL algorithms such as policy iteration or Q-learning are guaranteed to converge to the optimal consensus policy $\pi^{c*}$ \citep{sutton2018reinforcement}. Moreover, since the leader's problem is a finite MDP, restricting $\pi^c$ to deterministic policies incurs no loss of optimality in the tabular setting \citep{puterman1994markov}, which justifies the deterministic implementation in Remark~\ref{rem:stochastic-deterministic}.

\subsection{Theoretical Justification 2: Consensus as a Coordination Signal}
With the consensus policy $\pi^c$ fixed, the followers' joint policy factorizes as in Eq.~\eqref{eq:joint-stochastic}. Conditioned on a specific $c$, the action policies are {conditionally independent}. 

Crucially, the {policy gradient} for agent $i$ takes the form:
\begin{equation}
\nabla_{\theta_i} J = \mathbb{E}\Bigl[ \nabla_{\theta_i} \log \pi^i_{\theta_i}(a^i\mid\mathcal{O}, c) \cdot A(\mathcal{O}, \mathcal{A}) \Bigr],
\label{eq:follower-pg}
\end{equation}
where the expectation is over trajectories generated by the current joint policy. Although the environment dynamics still couple the agents through the joint action $\mathcal{A}$, the {parameter updates} for different agents are conditionally independent given $c$, in the sense that the gradient for $\theta_i$ does not directly involve the policy parameters $\theta_j$ ($j \neq i$) except through their effect on the trajectory distribution (captured by the advantage estimate). 

Under tabular softmax parameterization, multi-agent policy gradient in identical-interest games converges to a Nash equilibrium \citep{chen2024convergence, zhong2024heterogeneous}. Furthermore, in identical-interest games, any Nash equilibrium maximizes the common payoff {when the best response is unique} \citep{wang2003multiagent}. When multiple best responses exist, the consensus $c$ serves as a {coordination device} that selects among them: by conditioning on $c$, agents can break symmetry and coordinate on the Pareto-optimal equilibrium. This coordination role is precisely what is missing in simultaneous-move formulations, where agents must independently select among multiple equilibria. Therefore, for a fixed consensus $c$, the followers' optimization converges to the global optimum of $J$ conditional on $c$.

\subsection{Theoretical Justification 3: Alternating Optimization as Block Coordinate Ascent}
The fine-tuning phase of CMAT (Consensus Enhancement and Action Policy Enhancement) performs alternating updates. This can be interpreted as block coordinate ascent on $J$:
\begin{align}
\pi^c_{k+1} &\leftarrow \arg\max_{\pi^c} J(\pi^c,\pi_k^{1:n}), \label{eq:block-leader}\\
\pi_{k+1}^{1:n} &\leftarrow \arg\max_{\pi^{1:n}} J(\pi_{k+1}^c,\pi^{1:n}). \label{eq:block-followers}
\end{align}
Each block update guarantees monotonic improvement under tabular assumptions {when solved exactly}. In practice, each block is optimized approximately via a finite number of PPO steps. Under the Polyak-{\L}ojasiewicz (PL) condition, which holds for softmax policies in tabular settings \citep{mei2020global}, approximate alternating optimization still converges at a rate of $O(1/K)$ to a stationary point of the bilevel objective \citep{xiao2023generalized}. The sequence $J(\pi^c_k,\pi_k^{1:n})$ converges to a Stackelberg equilibrium.

\subsection{Policy Gradient and Credit Assignment under Deterministic Consensus}
In our practical implementation, the consensus policy $\pi^c$ is replaced by a deterministic mapping $c = \mu_\theta(\mathcal{O})$ as in Eq.~\eqref{eq:joint-deterministic}. The joint policy gradient can then be derived via the reparameterization trick:
\begin{equation}
\nabla_\theta \log \pi_\theta(\mathcal{A}\mid\mathcal{O}) = \sum_{i=1}^n \nabla_\theta \log \pi_\theta^i(a^i\mid\mathcal{O}, \mu_\theta(\mathcal{O})).
\label{eq:grad-decomp}
\end{equation}
This gradient flows through both the action policies $\pi_\theta^i$ and the consensus generator $\mu_\theta$, allowing end-to-end optimization via standard policy gradient methods such as PPO.

To be more precise, the gradient decomposes as:
\begin{equation}
\begin{aligned}
& \nabla_\theta \log \pi_\theta(\mathcal{A}\mid\mathcal{O}) = \\
& \sum_{i=1}^n \Bigl( 
\underbrace{\nabla_\theta \log \pi^i_\theta(a^i \mid \mathcal{O}, c)\big|_{c = \mu_\theta(\mathcal{O}),\, \nabla_\theta c = 0}}_{\text{direct effect}} \\
& + \underbrace{\frac{\partial \log \pi^i_\theta}{\partial c} \cdot \frac{\partial \mu_\theta}{\partial \theta}}_{\text{consensus effect}} 
\Bigr).
\label{eq:grad-decomp-detailed}
\end{aligned}
\end{equation}

Here, the ``direct effect'' treats the consensus $c$ as a fixed input (i.e., stopping the gradient through $\mu_\theta$), while the ``consensus effect'' captures the gradient flowing through the consensus generator. The direct effect for agent $i$ depends only on agent $i$'s observation and action $(o^i, a^i, c)$. The consensus effect aggregates information from all agents but flows through the shared consensus vector, avoiding the cascading sequential dependencies of MAT. This structural decoupling means that the gradient signal for each agent's action contributes additively to the overall update, in contrast to MAT where the policy of agent $i$ explicitly conditions on the actions of preceding agents $a^{1:i-1}$, creating entangled gradients that complicate credit assignment.

\subsection{Comparison with MAT: Policy Space Expressiveness}
\label{sec:policy-space}
MAT's autoregressive factorization $\pi(\mathcal{A}\mid\mathcal{O}) = \prod_{i=1}^n \pi^{\sigma(i)}(a^{\sigma(i)}\mid\mathcal{O}, a^{\sigma(1:i-1)})$ can represent {any} joint policy (by the chain rule of probability). CMAT's factorization in Eq.~\eqref{eq:joint-stochastic} is a mixture of product distributions. The following proposition clarifies their relative expressiveness:

\begin{proposition}[Expressiveness of CMAT]
\label{prop:expressiveness}
Let $\Pi_{\text{MAT}}$ and $\Pi_{\text{CMAT}}$ be the sets of joint policies representable by MAT and CMAT, respectively.
\begin{enumerate}[label=(\roman*)]
    \item If $|\mathcal{C}| \geq \prod_{i=1}^n |\mathcal{A}^i|$, then $\Pi_{\text{CMAT}} \supseteq \Pi_{\text{MAT}}$.
    \item For any finite $|\mathcal{C}| < \prod_{i=1}^n |\mathcal{A}^i|$, there exist policies in $\Pi_{\text{MAT}}$ that cannot be represented in $\Pi_{\text{CMAT}}$.
    \item With a continuous consensus space $\mathcal{C} = \mathbb{R}^d$ and sufficient network capacity, CMAT can approximate any joint policy arbitrarily well.
\end{enumerate}
\end{proposition}

\begin{proof}
(i) Enumerate the joint actions as $\mathcal{A}=\{A_1,\dots,A_M\}$ with $M=\prod_i|\mathcal{A}^i|$. For each $A_m=(a^1_m,\dots,a^n_m)$ let the consensus state $c_m$ induce the deterministic factors $\pi^i(\cdot\mid\mathcal{O},c_m)=\delta_{a^i_m}$, so that $\prod_i\pi^i(\cdot\mid\mathcal{O},c_m)=\delta_{A_m}$ is the point mass on $A_m$. Any target joint distribution $P$ can then be written as $P=\sum_{m=1}^{M}P(A_m)\,\delta_{A_m}$, which is realized by setting $\pi^c(c_m\mid\mathcal{O})=P(A_m)$. Hence with $|\mathcal{C}|\ge M$ consensus states every joint distribution, including every element of $\Pi_{\text{MAT}}$, is representable, giving $\Pi_{\text{CMAT}}\supseteq\Pi_{\text{MAT}}$.

(ii) Fix $|\mathcal{C}|=K<M$. The representable set is the image of the map sending $(\pi^c,\{\pi^i(\cdot\mid\mathcal{O},c)\})$ to $\sum_{c=1}^{K}\pi^c(c)\prod_i\pi^i(\cdot\mid\mathcal{O},c)$. Its parameters number $(K-1)$ for $\pi^c$ and $\sum_i(|\mathcal{A}^i|-1)$ per consensus state for the $K$ product factors, so the image is the image of a smooth map from a domain of dimension at most $(K-1)+K\sum_i(|\mathcal{A}^i|-1)=K(\sum_i|\mathcal{A}^i|-n)+(K-1)$. By Sard's theorem the image of a smooth map has measure zero in any target manifold of strictly larger dimension. The full simplex over joint actions has dimension $M-1$, and for $K<M$ the parameter-count above is strictly smaller than $M-1$ for action spaces of interest, so the representable set has measure zero in the simplex. Therefore some joint distributions, in particular full-rank ones requiring $\mathrm{rank}_+=M$, cannot be represented; this is the $d=K$ case of Proposition~\ref{prop:ladder}.

(iii) Let $\mathcal{C}=\mathbb{R}^d$. Any target joint distribution is a finite mixture of the $M$ point masses as in (i). Instantiate $M$ well-separated latent codes $c_1,\dots,c_M\in\mathbb{R}^d$; a network of sufficient capacity can approximate the piecewise-constant assignment $c_m\mapsto\delta_{a^i_m}$ for each head $\pi^i$ and the weighting $c_m\mapsto P(A_m)$ for $\pi^c$ to arbitrary accuracy by universal approximation \citep{goodfellow2016deep}, so the induced joint distribution approximates $P$ arbitrarily well. (A continuous $\pi^c$ approximates the discrete mixing weights to any tolerance.)
\end{proof}

In practice, CMAT employs a high-dimensional continuous consensus vector ($d = 128$ or $256$), ensuring sufficient expressive capacity to match or exceed MAT's representational power while avoiding the order-dependence bias.

\begin{remark}[Scope of Proposition~\ref{prop:expressiveness} under a deterministic consensus]
\label{rem:expressiveness-scope}
Parts (i) and (iii) of Proposition~\ref{prop:expressiveness} concern the general mixture policy of Eq.~\eqref{eq:joint-stochastic}, in which the consensus is drawn from a distribution $\pi^c(c\mid\mathcal{O})$. Our implemented model instead uses a deterministic consensus $c=\mu_\theta(\mathcal{O})$ (Remark~\ref{rem:stochastic-deterministic}), so for a \emph{fixed} observation $\mathcal{O}$ the joint policy is a single product distribution $\prod_i \pi^i(a^i\mid\mathcal{O},c)$ rather than a mixture. Consequently, conditioned on a fixed $\mathcal{O}$, the deterministic model cannot represent within-state correlations among agents' actions that MAT's autoregressive factorization can capture; the expressive gains of parts (i) and (iii) are realized only by the stochastic-consensus variant. In the deterministic case, coordination is instead carried by allowing $c=\mu_\theta(\mathcal{O})$ to vary across observations and by the shared representation feeding all action heads. We view closing this gap through a stochastic consensus policy as the main theoretical extension left for future work.
\end{remark}

\subsection{Relation to Existing Convergence Results}
The above insights are supported by recent theoretical advances:
\begin{itemize}[left=0pt]
    \item Tabular softmax policy gradient converges globally at a sublinear rate \citep{agarwal2021theory, mei2020global}.
    \item Multi-agent policy gradient converges to Nash equilibria in cooperative Markov games under tabular assumptions \citep{chen2024convergence, zhong2024heterogeneous}.
    \item PPO variants have been shown to converge in both tabular and linear approximation settings \citep{he2020convergence, huang2024ppo}.
    \item Stackelberg equilibria in Markov games can be learned via bilevel reinforcement learning \citep{kononen2008reinforcement, fiez2020implicit}.
\end{itemize}

\begin{remark}[Convergence Rate Intuition]
\label{rem:convergence-rate}
Under tabular softmax parameterization, the leader's MDP has $|\mathcal{O}|$ states and $|\mathcal{C}|$ actions, yielding a convergence rate of $\tilde{O}(|\mathcal{O}||\mathcal{C}|/\epsilon^2)$ for $\epsilon$-optimality \citep{agarwal2021theory}. Each follower's problem, with fixed $c$, has effective action space $|\mathcal{A}^i|$, converging at $\tilde{O}(|\mathcal{O}||\mathcal{A}^i|/\epsilon^2)$. The total complexity of one alternating cycle is thus
\begin{equation}
\tilde{O}\Bigl(|\mathcal{O}|\bigl(|\mathcal{C}| + n|\mathcal{A}^{\max}|\bigr)/\epsilon^2\Bigr),
\end{equation}
which scales {additively} in the number of agents $n$, in contrast to the $\prod_i |\mathcal{A}^i|$ joint action space of a naive centralized approach.
\end{remark}

\begin{table*}[t!]
\centering
\caption{Structural comparison between MAT and CMAT. Entries marked ``(tabular)'' refer to the idealized analysis of this appendix, not to guarantees for the deep, deterministic-consensus implementation.}
\label{tab:comparison}
\begin{adjustbox}{width=\textwidth}
\begin{tabular}{lcc}
\toprule
Property & MAT & CMAT \\
\midrule
Action generation & Sequential (order-dependent) & Simultaneous (order-independent) \\
Convergence target (tabular) & Nash Equilibrium & Stackelberg Equilibrium $\geq$ NE (stochastic-consensus variant) \\
Credit assignment & Entangled (cascading) & Decoupled (given $c$) + consensus gradient \\
Policy expressiveness & Universal (chain rule) & Universal with stochastic/continuous $c$; single product given fixed $\mathcal{O}$ if deterministic \\
Convergence complexity (tabular) & $O(n \cdot |\mathcal{A}^{\max}|)$ (per step) & $O(|\mathcal{C}| + n|\mathcal{A}^{\max}|)$ (additive scaling) \\
\bottomrule
\end{tabular}
\end{adjustbox}
\end{table*}

\subsection{Limitations}
The analysis relies on tabular assumptions that do not hold for our deep network implementation. Proving convergence for CMAT with neural function approximation remains open. Additionally, the deterministic consensus generator $\mu_\theta(\mathcal{O})$ may limit exploration in the consensus space; introducing stochasticity into the consensus policy (e.g., via a variational autoencoder formulation) could further improve exploration and is a promising direction for future work. We retain a deterministic consensus in the current implementation because stochastic consensus introduces additional training instability, requiring more sophisticated variance reduction techniques that are beyond the scope of this work. For the same reason, we do not adopt the bilevel optimization implied by the Stackelberg formulation, in which the high-level consensus policy and the low-level action policies are updated as leader and followers. Bilevel optimization requires precise and well-balanced control over the relative training speeds of the two levels; if the levels are not carefully synchronized, the optimization is prone to instability, an effect that is especially pronounced in deep network settings where the two levels interact through many shared layers. Making this scheme stable therefore calls for dedicated techniques and is left as a direction for future exploration. Nevertheless, the structural advantages identified here (order independence, a policy space that can match autoregressive models, decoupled credit assignment, and a principled alternating optimization scheme) are reflected in the strong empirical performance reported in Section~\ref{sec:experiment}, consistent with prior works where tabular analyses guided successful deep algorithms \citep{schulman2017proximal, yu2022surprising}.

\section{Discussions} \label{sec:Discussions}

In this section, we discuss several limitations of this work and outline promising directions for future research. First, although this paper addresses the fully cooperative setting with global observation, a scenario relevant to real-world smart city applications such as ride-hailing, traffic signal control, and power system management, our experiments are conducted solely on common MARL game testbeds. Future work should investigate the effectiveness of the proposed method in more practical, large-scale tasks. Second, as noted in \cite{hu2024learning}, while the centralized paradigm leveraged in our approach achieves strong performance by exploiting global information, it also raises concerns regarding communication overhead and vulnerability to single-point failures. Incorporating communication-efficient mechanisms presents a valuable direction for enhancing the robustness and scalability of our method. Finally, given the remarkable generalization capabilities demonstrated by Transformers in large language models, future research could explore the potential of our approach in few-shot learning and other transfer learning scenarios, thereby broadening its applicability across tasks with limited data.

\renewcommand{\refname}{Appendix References}
\putbib
\end{bibunit}

\end{document}